\documentclass[submission]{jsfds}
\usepackage[latin1]{inputenc}
\usepackage[T1]{fontenc}

\usepackage{float}
\usepackage{verbatim}
\usepackage{multirow}
\usepackage{array}
\usepackage{enumerate}

\usepackage{bm}

\usepackage[english]{babel}

\usepackage[margin=1.1in]{geometry}

\def\non{\noindent}
\def\vk{\vskip0.1cm}

\usepackage{natbib} 
\graphicspath{{figures/}}

\newcommand{\bz}{\mathbf{z}}

\newcommand{\bsw}{\boldsymbol{w}}

\newcommand{\bsx}{\boldsymbol{x}}
\newcommand{\bsX}{\boldsymbol{X}}
\newcommand{\bsy}{\boldsymbol{y}}
\newcommand{\bsY}{\boldsymbol{Y}}

\newcommand{\cX}{\mathcal{X}}

\newcommand{\cD}{\mathcal{D}}

\newcommand{\bsbeta}{\boldsymbol{\beta}}

\newcommand{\bstheta}{\boldsymbol{\theta}}

\newcommand{\E}{\mathbb{E}}

\newcommand{\Pro}{\mathbb{P}}
\newcommand{\R}{\mathbb{R}}

\newcommand{\N}{\mathbb{N}}

\startlocaldefs
\endlocaldefs

\newtheorem{theorem}{Theorem}[section]
\newtheorem{lemma}{Lemma}[section]

\theoremstyle{definition}
\newtheorem{definition}{Definition}[section]

\setmainlanguageenglish
\firstpage{1}
\begin{document}
\begin{frontmatter}

\title{Regularized Maximum Likelihood Estimation and Feature Selection in Mixtures-of-Experts Models}
\runtitle{Regularized Estimation and Feature Selection in Mixtures of Experts}
\alttitle{Estimation par maximum de vraisemblance r\'egularis\'e et s\'election de variables dans les mod{\`e}les de m\'elanges d'experts}

\begin{aug}
  \auteur{%
  \prenom{Faicel}
  \nom{Chamroukhi}%
    \thanksref{t1}
    \contact[label=e1]{faicel.chamroukhi@unicaen.fr}}
  \and%
  \auteur{%
  \prenom{Bao-Tuyen}
  \nom{Huynh}%
  \thanksref{t2}
  \contact[label=e2]{bao-tuyen.huynh@unicaen.fr}}
  \affiliation[t1]{Normandie Univ, UNICAEN, UMR CNRS LMNO, Dpt of Mathematics and Computer Science, 14000 Caen, France\\ \printcontact{e1}}

  \affiliation[t2]{Normandie Univ, UNICAEN, UMR CNRS LMNO, Dpt of Mathematics and Computer Science, 14000 Caen, France. \\ 
  \printcontact{e2}} 
  \runauthor{Chamroukhi, Huynh}
\end{aug}

\begin{abstract}
Mixture of Experts (MoE) are successful models for modeling heterogeneous data in many statistical learning problems including regression, clustering and classification. Generally fitted by maximum likelihood estimation via the well-known EM algorithm, their application to high-dimensional problems is still therefore challenging. We consider the problem of fitting and feature selection in MoE models, and propose a regularized maximum likelihood estimation approach that encourages sparse solutions for heterogeneous regression data models with potentially high-dimensional predictors. 
Unlike state-of-the art regularized MLE for MoE, the proposed modelings do not require an approximate of the penalty function. 
We develop two hybrid EM algorithms: an Expectation-Majorization-Maximization (EM/MM) algorithm, and an EM algorithm with coordinate ascent algorithm.  The proposed algorithms allow to automatically obtaining sparse solutions without thresholding, and avoid matrix inversion by allowing univariate parameter updates. An experimental study shows the good performance of the algorithms in terms of recovering the actual sparse solutions, parameter estimation, and clustering of heterogeneous regression data.
\end{abstract}


\begin{altabstract}
Les m\'elanges d'experts (MoE) sont des mod\`eles efficaces pour la mod\'elisation de donn\'ees h\'et\'erog\`enes dans de nombreux probl\`emes en apprentissage statistique, y compris en r\'egression, en classification et en discrimination. G\'en\'eralement ajust\'es par maximum de vraisemblance via l'algorithme EM, leur application aux probl\'emes de grande dimension est difficile dans un tel contexte. Nous consid\'erons le probl\`eme de l'estimation et de la s\'election de variables dans les mod\`eles de m\'elanges d'experts, et proposons une approche d'estimation par maximum de vraisemblance r\'egularis\'e qui encourage des solutions parcimonieuses pour des mod\'eles de donn\'ees de r\'egression h\'et\'erog\`enes comportant un nombre de pr\'edicteurs potentiellement grand.
La m\'ethode de r\'egularisation propos\'ee, contrairement  aux m\'ethodes 
 de l'\'etat de l'art sur les m\'elanges d'experts, ne se base pas sur une p\'enalisation approch\'ee et ne n\'ecessite pas de seuillage pour retrouver la solution parcimonieuse. 
L'estimation parcimonieuse des param\`etres s'appuie sur une r\'egularisation de l'estimateur du maximum de vraisemblance pour les experts et les fonctions d'activations, mise en {\oe}uvre par deux versions d'un algorithme EM hybride.  
L'\'etape M de l'algorithme, effectu\'ee par mont\'ee de coordonn\'ees ou par un algorithme MM, 
\'evite l'inversion de matrices dans la mise \`a jour et rend  ainsi prometteur le passage de l'algorithme \`a l'\'echelle. Une \'etude exp\'erimentale 
 met en \'evidence de bonnes performances de l'approche propos\'ee. 
\end{altabstract}

\begin{altkeywords}
\kwd{M\'elanges d'experts} 
\kwd{Classification \'a base de mod\'ele}
\kwd{S\'election de variable}
\kwd{R\'egularisation}
\kwd{Algorithme EM} 
\kwd{Mont\'ee de coordonn\'ees}
\kwd{Algorithme MM}
\kwd{Donn\'ees de grande dimension}
\end{altkeywords}

\begin{AMSclass}
\kwd{62-XX}
\kwd{62H30}
\kwd{62G05}
\kwd{62G07}
\kwd{62H12}
\kwd{62-07}
\kwd{62J07}
\kwd{68T05}
\end{AMSclass}
\end{frontmatter}

\section{Introduction}

Mixture of experts (MoE) models introduced by \cite{Jac91} are successful for modeling heterogeneous data in statistics and machine learning problems including regression, clustering and classification.  MoE belong to the family of mixture models \citep{TitteringtonBookMixtures,McLachlan2000FMM,SylviaFruhwirthBook2006}  and is a fully conditional mixture model where both the mixing proportions, i.e, the gating network, and the components densities, i.e, the experts network, depend on the inputs. 
A general review of the MoE models and  their applications can be found in \cite{NguyenChamroukhi-MoE}. 
While the MoE modeling with maximum likelihood estimation (MLE) is widely used, its application in high-dimensional problems is still challenging due to the well-known problem of the ML estimator in such a setting. Indeed,  in high-dimensional setting, the features can be correlated and thus the actual features that explain the problem reside in a low-dimensional space. Hence, there is a need to select a subset of the potentially large number of features, that really explain the data. 
To avoid singularities and degeneracies of the MLE as highlighted namely in \cite{Stephens1997,Snoussi2005,Fraley2005, Fraley2007}, one can regularize the likelihood through a prior distribution over the model parameter space. 
A better fitting can therefore be achieved by regularizing the objective function so that to encourage sparse solutions.  
However, feature selection by regularized inference encourages sparse solutions, while having a reasonable computational cost. Several approaches have been proposed to deal with the feature selection task, both in regression and in clustering. 

For regression, the well-known Lasso method \citep{Tib96}  is one of the most popular and successful regularization technique which utilizes the $\ell_1$ penalty to regularize the squared error function, or by equivalence the log-likelihood in Gaussian regression, and to achieve parameter estimation and feature selection. This allows to shrink coefficients toward zero, and can also set many coefficients to be exactly zero. 
%
%
While the problem of feature selection and regularization is more popular in this supervised learning context, it has took an increasing interest in the unsupervised context, namely in clustering, as in \cite{Witten2010} where a sparse $K$-means algorithm is introduced for clustering high-dimensional data using a Lasso-type penalty to select the features, including in model-based clustering.  
In that context, \cite{Pan2007} considered the problem of fitting mixture of Gaussians by maximizing a penalized log-likelihood with an $\ell_1$ penalty over the mean vectors. This allows to shrink some variables in the mean vectors to zero and to provide a sparse mixture model with respect to the means and thus to perform the clustering in a low-dimensional space. 
\cite{Maugis2009b} proposed the SRUW model, by relying on the role of the variables in clustering and by distinguishing between relevant variables and irrelevant variables to clustering. In this approach,
the feature selection problem is considered as a model selection problem for model-based clustering, by maximizing a BIC-type criterion given a collection of models. The drawback of this approach is that it is time demanding for high-dimensional data sets. To overcome this drawback, \cite{Celeux2017} proposed an alternative variable selection procedure in two steps. First, the variables are ranked through a Lasso-like procedure, by an $\ell_1$ penalties for the mean and the covariance matrices. Then their roles are determined by using the SRUW model. 
Other interesting approaches for feature selection in model-based clustering for high-dimensional data can be found in \cite{Law04,Raftery06,Maugis2009a}.\vk
In related mixture models for simultaneous regression and clustering, including mixture of linear regressions (MLR), where the mixing proportions are constant, 
 \cite{Kha07} proposed regularized ML inference, including MIXLASSO, MIXHARD and MIXSCAD and provided asymptotic properties corresponding to these penalty functions. Another $\ell_1$ penalization for MLR models for high-dimensional data was proposed by \cite{Sta10} which uses an adaptive Lasso penalized estimator. 
 An efficient EM algorithm with provable convergence properties has been introduced for the optimization
variable selection.
\cite{Mey13} provided an $\ell_1$-oracle inequality for a Lasso estimator in finite mixture of Gaussian regression models. 
 This result can be seen as a complementary result to \cite{Sta10}, by studying the $\ell_1$-regularization properties of the Lasso in parameter estimation, rather than by considering it as a variable selection procedure. 
 This work was extended later in \cite{Dev15} by considering a mixture of multivariate Gaussian regression models. When the set of features can be structued in the form of groups, \cite{Hui15} introduced the two types of penalty functions called MIXGL1 and MIXGL2 for MLR models, based on the structured regularization of the group Lasso. %
 A MM algorithm \cite{Lan13} for MLR with Lasso penalty can be found in \cite{Llo16}, which allows to avoid matrix operations.
 In \cite{Kha10}, the author extended his MLR regularization to the MoE setting, and provided a root-$n$ consistent and oracle properties for Lasso and SCAD penalties, and developed an EM algorithm for fitting the models. However, as we will discuss it in Section \ref{sec: RMoE}, this is based on approximated penalty function, and uses a Newton-Raphson procedure in the updates of the gating network parameters, and thus requires matrix inversion. 

In this paper, we consider the regularized MLE and clustering in MoE models as in \cite{Kha10}. We propose a new regularized maximum likelihood estimation approach with two hybrid algorithms for maximizing the proposed objective function. The proposed algorithms for fitting the model consist of an Expectation-Majorization-Maximization (EMM) algorithm and an EM algorithm with a coordinate ascent algorithm.   
The proposed approach does not require an approximate of the regularization term, and the two developed hybrid algorithms, allow to automatically select sparse solutions without thresholding. 

The remainder of this paper is organized as follows. 
In Section \ref{sec: MoE} we present the modeling with MoE for heterogeneous data. Then, in Section \ref{sec: RMoE}, we present, the regularized maximum likelihood strategy of the MoE model, and the two proposed EM-based algorithms. 
An experimental study, carried out on simulated and two real data sets, are given in Section \ref{sec: Experiments}. In Section \ref{sec:Discuss}, we discuss the effectiveness of our method in dealing with moderate dimensional problems, and consider an experiment which promotes its use in high-dimensional scenarios. Finally, in Section \ref{Sec:Con}, we draw concluding remarks and mention future direction.

\section{Modeling with Mixture of Experts (MoE)}
\label{sec: MoE}

Let $((\bsX_1,\bsY_1),\ldots,(\bsX_n,\bsY_n))$ be a random sample of $n$  independently and identically distributed (i.i.d) pairs $(\bsX_i,\bsY_i)$, ($i=1,\ldots, n$) 
where $Y_i\in \cX \subset \R^d$ is the $i$th response  given some vector of predictors $\bsX_i \in \cX \subset \R^p$.
We consider the MoE modeling  for the analysis of a heteregeneous set of such data.
Let $\cD = ((\bsx_1,y_1),\ldots,(\bsx_n,y_n))$ be an observed data sample. 

\subsection{The model}
The mixture of experts model assumes that the observed pairs $(\bsx,\bsy)$ are generated from $K\in \N$ (possibly unknown) tailored probability density components (the experts) governed by a hidden categorical random variable $Z\in [K]=\{1,\ldots,K\}$ that indicates the component from which a particular observed pair is drawn. The latter represents the gating network. 
Formally, the gating network is defined by the distribution of the hidden variable $Z$ given the predictor $\bsx$, i.e., $\pi_k(\bsx;\bsw) = \Pro(Z=k|\bsX=\bsx;\bsw)$, which is in general given by gating softmax functions of the form:
\begin{eqnarray}
\pi_k(\bsx_i;\bsw) = \Pro(Z_i=k|\bsX_i=\bsx_i;\bsw)
 = \frac{\exp(w_{k0}+\bsx_i^T \bsw_k)}{1 + \sum\limits_{l=1}^{K-1}\exp(w_{l0}+\bsx_i^T \bsw_l)}
\label{eq.softmax}
\end{eqnarray}for $k=1,\hdots,K-1$ with $(w_{k0}, \bsw^T_k) \in \R^{p+1}$ and  $(w_{K0}, \bsw^T_K)=(0,\mathbf{0})$ for identifiability \cite{Jia99}.
The experts network is defined by the conditional densities  $f(\bsy_i|\bsx_i;\bstheta_k)$ which is the short notation of $f(\bsy_i|\bsX=\bsx,Z_i=k;\bstheta)$. 
The MoE thus decomposes the probability density of the observed data as a convex sum of a finite experts weighted by a softmax gating network, and  can be defined by the following semi-parametric probability density (or mass) function: %
\begin{equation}
f(\bsy_i|\bsx_i; \bstheta) = \sum_{k=1}^K \pi_k(\bsx_i;\bsw) f(\bsy_i|\bsx_i;\bstheta_k)
\label{eq.MoE}
\end{equation}that is parameterized by the parameter vector defined by 
$\bstheta = (\bsw^T_1,\ldots,\bsw^T_{K-1},\bstheta^T_1,\ldots,\bstheta^T_K)^T \in \R^{\nu_{\bstheta}}$ ($\nu_{\bstheta} \in \N$) 
where 
 $\bstheta_k$ ($k=1,\ldots,K$) is the parameter vector of the $k$th expert. 
\\
%
For a complete account of MoE, types of gating networks and experts networks, the reader is referred to \cite{NguyenChamroukhi-MoE}. 

The generative process of the data assumes the following hierarchical representation. 
First, given the predictor $\bsx_i$, the categorical variable $Z_i$ follows the multinomial distribution:
\begin{equation}
Z_i|\bsx_i \sim \text{Mult}(1;\pi_1(\bsx_i;\bsw),\hdots,\pi_K(\bsx_i;\bsw)) 
\label{eq.Z generation in MoE}
\end{equation}where each of the probabilities $\pi_{z_i}(\bsx_i;\bsw) = \mathbb{P}(Z_i = z_i|\bsx_i)$ is given by the multinomial logistic function (\ref{eq.softmax}). 
Then, conditional on the hidden variable $Z_i = z_i$, given the covariate $\bsx_i$, a random variable $Y_i$ is assumed to be generated according to the following representation
\begin{equation}
\bsY_i|Z_i = z_i, \bsX_i=\bsx_i   \sim p(\bsy_i|\bsx_i;\bstheta_{z_i})
\label{eq.MoE generative model}
\end{equation}where $p(\bsy_i|\bsx_i;\bstheta_k) = p(\bsy_i|Z_i = z_i,\bsX_i=\bsx_i;\bstheta_{z_i})$ is the probability density or the probability mass function of the expert $z_i$ depending on the nature of the data ($\bsx,\bsy$) within the group $z_i$. 
In the following, we consider MoE models for regression and clustering of continuous data. 
%
%
Consider the case of univariate continuous outputs $Y_i$. A common choice to model the relationship between the input $\bsx$ and the output $Y$ is by considering regression functions. 
Thus, within each homogeneous group $Z_i=z_i$, the response $Y_i$, given the expert $k$, is modeled by the  noisy linear model:
$Y_i = \beta_{z_i0} + \bsbeta^T_{z_i}\bsx_i + \sigma_{z_i}\varepsilon_i$,
where the $\varepsilon_i$ are 
standard i.i.d zero-mean unit variance Gaussian noise variables, the bias coefficient
$\bsbeta_{k0}\in \R$ and $\bsbeta_k \in \R^p$ are the usual unknown regression coefficients describing the expert $Z_i=k$, 
and $\sigma_k >0$ corresponds to the standard deviation of the noise. In such a case, 
the generative model (\ref{eq.MoE generative model})  of $Y$ becomes
\begin{equation}
Y_i|Z_i = z_i,\bsx_i \sim 
\mathcal{N}(.;\beta_{z_i0} + \bsbeta^T_{z_i}\bsx_i, \sigma_{z_i}^2)\cdot
\label{eq.NMoE regression generative model}
\end{equation}

\subsection{Maximum likelihood parameter estimation}
\label{ssec: EM-MoE}
Assume that, $\cD = ((\bsx_1,\bsy_1),\ldots,(\bsx_n, \bsy_n))$ is an observed data sample generated from the MoE (\ref{eq.MoE}) with unknown parameter $\bstheta$. The parameter vector $\bstheta$ is commonly estimated by maximizing the observed data log-likelihood
$\log L(\bstheta) = \sum_{i=1}^{n}\log \sum_{k=1}^{K} \pi_k(\bsx_i;\bsw) f(\bsy_{i}|\bsx_i;\bstheta_{k})$ 
by using the EM algorithm \citep{Dem77, Jac91} 
which allows to iteratively find an appropriate local maximizer of the log-likelihood function.
%
In the considered model for Gaussian regression,  the maximized log-likelihood is given by
\begin{equation}
\log L(\bstheta) = \sum_{i=1}^n\log\Bigl[\sum_{k=1}^K \pi_k(\bsx_i;\bsw)\mathcal{N}(y_i; \beta_{k0} +\bsbeta^T_k \bsx_i,\sigma_k^2)\Bigl].
\label{eq.log-lik NMoE regression}
\end{equation}
However, it is well-known that the MLE may be unstable of even infeasible in high-dimension namely due to possibly redundant and correlated features. In such a context, a regularization of the MLE is needed.

%

\section{Regularized Maximum Likelihood parameter Estimation of the MoE}
\label{sec: RMoE}
Regularized maximum likelihood estimation allows the selection of a relevant subset of features for prediction and thus encourages sparse solutions. In mixture of experts modeling, one may consider both sparsity in the feature space of the gates, and of the experts. 
We propose to infer the MoE model by maximizing a regularized log-likelihood criterion, which encourages sparsity for both the gating network parameters and the experts network parameters, and does not require any approximation, along with performing the maximization, so that to avoid matrix inversion. 
The proposed regularization combines a Lasso penalty for the experts parameters, and an elastic net like penalty for the gating network, defined by:
\begin{equation}
\label{eq:PenLoglik MoE}
PL(\bstheta) = L(\bstheta) - \sum_{k=1}^K\lambda_k\|\bsbeta_k\|_1 - \sum_{k=1}^{K-1}\gamma_k\|\bsw_k\|_1 - \frac{\rho}{2}\sum_{k=1}^{K-1}\|\bsw_k\|_2^2.
\end{equation}
A similar strategy has been proposed in  \cite{Kha10} where the author proposed a regularized ML function like (\ref{eq:PenLoglik MoE}) but which is then approximated in the EM algorithm of the model inference. The EM algorithm for fitting the model follows indeed the suggestion of \cite{Fan01} to approximate the penalty function 
in a some neighborhood 
by a local quadratic function. Therefore, a Newton-Raphson can be used to update parameters in the M-step. The weakness of this scheme is that once a feature is set to zero, it may never reenter the model at a later stage of the algorithm.  
To avoid this numerical instability of the algorithm due to the small values of some of the features in the denominator of this approximation, \cite{Kha10} replaced that  approximation 
 by an $\epsilon$-local quadratic function. %
Unfortunately, these strategies have some drawbacks. First, by approximating the penalty functions with ($\epsilon$-)quadratic functions, none of the components will be exactly zero. Hence, a threshold should be considered to declare a coefficient is zero, and this threshold affects the degree of sparsity. Secondly, it cannot guarantee the non-decreasing property of the EM algorithm of the penalized objective function. 
Thus, the convergence of the EM algorithm cannot be ensured. One has also to choose $\epsilon$ as an additional tuning parameter in practice. Our proposal overcomes these limitations. 

The $\ell_2$ term penalty is added in our model to take into account  possible strong correlation between
the features $x_j$ which could be translated especially on the coefficients of the gating network $\bsw$ because they are related between the different experts, contrary to the regression coefficients $\bsbeta$.
The resulting combination of $\ell_1$ and $\ell_2$ for $\bsw$ leads to an elastic net-like regularization, which enjoys similar sparsity of representation as the $\ell_1$ penalty.
The $\ell_2$ term is not however essential especially when the main goal is to retrieve the sparsity, rather than to perform prediction.

\subsection{Parameter estimation with block-wise EM}
\label{ssec: EM}
We propose two block-wise EM algorithms to monotonically find at least local maximizers of (\ref{eq:PenLoglik MoE}).
The E-step is common to both algorithms, while in the M-step, two different algorithms are proposed to update the model parameters. More specifically, the first one relies on a MM algorithm, while the second one uses a coordinate ascent to update the gating network $\bsw$ parameters  and the experts network $\bsbeta$' parameters.
The EM algorithm for the maximization of (\ref{eq:PenLoglik MoE}) firstly requires the construction of  the penalized complete-data log-likelihood
{\begin{equation}
\log PL_c(\bstheta) = \log L_c(\bstheta)- \sum_{k=1}^K\lambda_k\|\bsbeta_k\|_1 - \sum_{k=1}^{K-1}\gamma_k\|\bsw_k\|_1 - \frac{\rho}{2}\sum_{k=1}^{K-1}\|\bsw_k\|_2^2
\label{eq:complete log-lik RMoE}
\end{equation}}where
$\log L_c(\bstheta) = \sum_{i=1}^{n}\sum_{k=1}^{K} Z_{ik} \log \left[\pi_k(\bsx_i;\bsw) f(\bsy_i|\bsx_i;\bstheta_k)\right]$ 
is the standard complete-data log-likelihood, $Z_{ik}$ is an indicator binary-valued variable such that $Z_{ik}=1$ if $Z_i=k$ (i.e., if the $i$th pair $(\bsx_i,\bsy_i)$ is generated from the $k$th expert component) and $Z_{ik}=0$ otherwise.
Thus, the EM algorithm for the RMoE in its general form runs as follows. After starting with an initial solution $\bstheta^{[0]}$, it alternates between the two following steps until convergence (e.g., when there is no longer a significant change in the relative variation of the regularized log-likelihood).
\subsection{E-step}
\label{ssec: E-step RMoE} The E-Step computes the conditional expectation of the penalized complete-data log-likelihood (\ref{eq:complete log-lik RMoE}),  given the observed data $\cD$ and a current parameter vector $\bstheta^{[q]}$, $q$ being the current iteration number of the block-wise EM algorithm: 
{\begin{align}
Q(\bstheta;\bstheta^{[q]}) &=  \E\left[\log PL_c(\bstheta)|\cD;\bstheta^{[q]}\right] \nonumber\\
&= \sum_{i=1}^{n}\sum_{k=1}^{K}\tau_{ik}^{[q]} \log \left[\pi_k(\bsx_i;\bsw) f_k(\bsy_{i}|\bsx_i;\bstheta_{k})\right]
- \sum_{k=1}^K\lambda_k\|\bsbeta_k\|_1 - \sum_{k=1}^{K-1}\gamma_k\|\bsw_k\|_1 - \frac{\rho}{2}\sum_{k=1}^{K-1}\|\bsw_k\|_2^2
\label{eq:Q-function RMoE}
\end{align}}where
{\begin{eqnarray}
\tau_{ik}^{[q]} = \Pro(Z_i=k|\bsy_{i},\bsx_i;\bstheta^{[q]})
= \frac{\pi_k(\bsx_i;\bsw^{[q]}) \mathcal{N}(y_i; \beta_{k0}^{[q]} +\bsx^T_i \bsbeta_k^{[q]},\sigma_k^{[q]2})}{\sum\limits_{l=1}^K \pi_l(\bsx_i;\bsw^{[q]})\mathcal{N}(y_i; \beta_{l0}^{[q]} + \bsx^T_i\bsbeta_l^{[q]},\sigma_l^{[q]2})}
\label{eq:RMoE post prob}
\end{eqnarray}}is the conditional probability that the data pair $(\bsx_i,\bsy_i)$ is generated by  the $k$th expert. 
This step therefore only requires  the computation of the conditional component probabilities $\tau^{[q]}_{ik}$ $(i=1,\ldots,n)$ for each of the $K$ experts. 
\subsection{M-step} The M-Step updates the parameters by maximizing the $Q$ function (\ref{eq:Q-function RMoE}), which can be written as
\begin{equation}
Q(\bstheta;\bstheta^{[q]}) = Q(\bsw;\bstheta^{[q]}) + Q(\bsbeta, \sigma;\bstheta^{[q]})
\end{equation}
with
{\begin{equation}\label{QP}
 Q(\bsw;\bstheta^{[q]}) = \sum_{i=1}^n\sum_{k=1}^K\tau_{ik}^{[q]}\log\pi_k(\bsx_i;\bsw)-  \sum_{k=1}^{K-1}\gamma_k\|\bsw_k\|_1 - \frac{\rho}{2}\sum_{k=1}^{K-1}\|\bsw_k\|_2^2,
\end{equation}}
and
{\begin{equation}\label{QnP}
Q(\bsbeta, \sigma;\bstheta^{[q]}) =  \sum_{i=1}^n\sum_{k=1}^K\tau_{ik}^{[q]}\log\mathcal{N}(y_i; \beta_{k0} +\bsx_i^T\bsbeta_k,\sigma_k^2) - \sum_{k=1}^K\lambda_k\|\bsbeta_k\|_1.
\end{equation}}The parameters $\bsw$ are therefore separately updated by maximizing the function 
{\begin{equation}
Q(\bsw;\bstheta^{[q]}) =
\sum_{i=1}^n\sum_{k=1}^{K-1}\tau_{ik}^{[q]}(w_{k0}+\bsx_i^T \bsw_k)- \sum_{i=1}^n\log\Bigl[1+\sum_{k=1}^{K-1}e^{w_{k0}+\bsx_i^T \bsw_k}\Bigl]
 - \sum_{k=1}^{K-1}\gamma_k\|\bsw_k\|_1- \frac{\rho}{2}\sum_{k=1}^{K-1}\|\bsw_k\|_2^2.
\label{eq.Qw}
\end{equation}}We propose and compare two approaches for maximizing (\ref{QP}) based on a MM algorithm and a coordinate ascent algorithm. These approaches have some advantages since they do not use any approximate for the penalty function, and have a separate structure which avoid matrix inversion. 

\subsubsection{MM algorithm for updating the gating network}
In this part, we construct a MM algorithm to iteratively update the gating network parameters $(w_{k0}, \bsw_k)$. At each iteration step $s$ of the MM algorithm, we maximize a minorizing function of the initial function (\ref{eq.Qw}). 
We begin this task by giving the definition of a minorizing function.
\begin{definition} (see \cite{Lan13})
Let $F(x)$ be a function of $x$. A function $G(x|x_m)$ is called a minorizing function of $F(x)$ at $x_m$ iff
$$F(x)\ge G(x|x_m) \text{ and } F(x_m) = G(x_m|x_m),\ \forall x.$$
\end{definition}\non
In the maximization step of the MM algorithm, we maximize the surrogate function $G(x|x_m)$, rather than the function $F(x)$ itself. If $x_{m+1}$ is the maximum of $G(x|x_m)$, then we can show that the MM algorithm forces $F(x)$ uphill, because
$$F(x_m) = G(x_m|x_m)\le G(x_{m+1}|x_m)\le F(x_{m+1}).$$
By doing so, we can find a local maximizer of $F(x)$. If $G(x_m|x_m)$ is well constructed, then we can avoid matrix inversion when maximizing it. Next, we derive the surrogate function for $Q(\bsw;\bstheta^{[q]})$. We start by the following lemma. 
\begin{lemma}\label{LL} If $x > 0$, then the function $f(x) = -\ln(1+x)$ can be minorized by
$$g(x|x_m) = -\ln(1+x_m) - \frac{x-x_m}{1+x_m},\ \text{ at }  x_m>0.$$ 
\end{lemma}\non
By applying this lemma and following \cite[page 211]{Lan13} we have
\begin{theorem}The function $I_1(\bsw) = - \sum\limits_{i=1}^n\log\Bigl[1+\sum\limits_{k=1}^{K-1}e^{w_{k0}+\bsx_i^T \bsw_k}\Bigl]$ is a majorizer of
$$G_1(\bsw|\bsw^{[s]}) = \sum\limits_{i=1}^n\Bigl[-\sum\limits_{k=1}^{K-1}\frac{\pi_k(\bsx_i;\bsw^{[s]})}{p+1}\sum\limits_{j=0}^p e^{(p+1)x_{ij}(w_{kj}-w_{kj}^{[s]})} - \log C_i^{m}+ 1- \frac{1}{C_i^{m}}\Bigl],$$
where $C_i^{m} = 1+\sum\limits_{k=1}^{K-1}e^{w_{k0}^{[s]}+\bsx_i^T \bsw_k^{[s]}}$ and $x_{i0}=1$.
\end{theorem}
\begin{proof}Using Lemma \ref{LL}, $I_{1i}(w) = -\log\Bigl[1+\sum\limits_{k=1}^{K-1}e^{w_{k0}+\bsx_i^T \bsw_k}\Bigl]$ can be minorized by 
\begin{align*}
G_i(\bsw|\bsw^{[s]}) & = -\log\Bigl[1+\sum\limits_{k=1}^{K-1}e^{w_{k0}^{[s]}+\bsx_i^T \bsw_k^{[s]}}\Bigl] - \frac{\sum\limits_{k=1}^{K-1}(e^{w_{k0}+\bsx_i^T \bsw_k}-e^{w_{k0}^{[s]}+\bsx_i^T \bsw_k^{[s]}})}{1+\sum\limits_{k=1}^{K-1}e^{w_{k0}^{[s]}+\bsx_i^T \bsw_k^{[s]}}}\\
\ & = -\log C_i^{m}+ 1-\frac{1}{C_i^{m}}-\sum\limits_{k=1}^{K-1}\frac{e^{w_{k0}^{[s]}+\bsx_i^T \bsw_k^{[s]}}}{C_i^{m}}e^{(w_{k0}+\bsx_i^T\bsw_k)-(w_{k0}^{[s]}+\bsx_i^T\bsw_k^{[s]})}\cdot
\end{align*}Now, by using arithmetic-geometric mean inequality then
\begin{equation}\label{AGinequality} e^{(w_{k0}+\bsx_i^T \bsw_k)-(w_{k0}^{[s]}+\bsx_i^T \bsw_k^{[s]})}= \prod\limits_{j=0}^p e^{x_{ij}(w_{kj}-w_{kj}^{[s]})}
\le \frac{\sum\limits_{j=0}^p e^{(p+1)x_{ij}(w_{kj}-w_{kj}^{[s]})}}{p+1}\cdot
\end{equation}When $(w_{k0}, \bsw_k) = (w_{k0}^{[s]}, \bsw_k^{[s]})$ the equality holds.\\
Thus, $I_{1i}(w)$ can be minorized by
\begin{align*} 
G_{1i}(\bsw|\bsw^{[s]}) &= -\sum\limits_{k=1}^{K-1}\frac{e^{w_{k0}^{[s]}+\bsx_i^T\bsw_k^{[s]}}}{(p+1)C_i^{m}}\sum\limits_{j=0}^p e^{(p+1)x_{ij}(w_{kj}-w_{kj}^{[s]})} - \log C_i^{m}+ 1- \frac{1}{C_i^{m}}\\
\ &= -\sum\limits_{k=1}^{K-1}\frac{\pi_k(\bsx_i;\bsw^{[s]})}{p+1}\sum\limits_{j=0}^p e^{(p+1)x_{ij}(w_{kj}-w_{kj}^{[s]})} - \log C_i^{m}+ 1- \frac{1}{C_i^{m}}\cdot
\end{align*}
This leads us to the minorizing function $G_1(\bsw|\bsw^{[s]})$ for $I_1(w)$
 $$G_1(\bsw|\bsw^{[s]}) = \sum\limits_{i=1}^n\Bigl[-\sum\limits_{k=1}^{K-1}\frac{\pi_k(\bsx_i;\bsw^{[s]})}{p+1}\sum\limits_{j=0}^p e^{(p+1)x_{ij}(w_{kj}-w_{kj}^{[s]})} - \log C_i^{m}+ 1- \frac{1}{C_i^{m}}\Bigl]\cdot$$
\end{proof}\non
Therefore, the minorizing function $G^{[q]}(\bsw|\bsw^{[s]})$ for $Q(\bsw;\bstheta^{[q]})$ is given by
$$G^{[q]}(\bsw|\bsw^{[s]}) = \sum\limits_{i=1}^n\sum\limits_{k=1}^{K-1}\tau_{ik}^{[q]}(w_{k0}+\bsx_i^T \bsw_k) + G_1(\bsw|\bsw^{[s]}) - \sum\limits_{k=1}^{K-1}\gamma_k\sum\limits_{j=1}^p|w_{kj}|- \frac{\rho}{2}\sum\limits_{k=1}^{K-1}\sum\limits_{j=1}^p w_{kj}^2.$$ 
Now, let us separate $G^{[q]}(\bsw|\bsw^{[s]})$ into each parameter for all $k \in\{1,\hdots,K-1\},\ j \in\{1,\hdots,p\}$, we have:
\begin{align}
G^{[q]}(w_{k0}|\bsw^{[s]}) &=\sum\limits_{i=1}^n \tau_{ik}^{[q]}w_{k0}- \sum\limits_{i=1}^n \frac{\pi_k(\bsx_i;\bsw^{[s]})}{p+1}e^{(p+1)(w_{k0}- w_{k0}^{[s]})}+A_k(\bsw^{[s]}),\label{wk0}\\
G^{[q]}(w_{kj}|\bsw^{[s]}) &= \sum\limits_{i=1}^n \tau_{ik}^{[q]}x_{ij}w_{kj}- \sum\limits_{i=1}^n \frac{\pi_k(\bsx_i;\bsw^{[s]})}{p+1}e^{(p+1)x_{ij}(w_{kj}- w_{kj}^{[s]})}-\gamma_k|w_{kj}|-\frac{\rho}{2}w_{kj}^2+B_{kj}(\bsw^{[s]}),
\end{align}
where $A_k(\bsw^{[s]})$ and $B_{kj}(\bsw^{[s]})$ are only functions of $\bsw^{[s]}$.\\
The update of $w_{k0}^{[s]}$ is straightforward  by maximizing (\ref{wk0}) and given by
\begin{equation}
w_{k0}^{[s+1]} = w_{k0}^{[s]} + \frac{1}{p+1}\ln\left(\frac{\sum\limits_{i=1}^n\tau_{ik}^{[q]}}{\sum\limits_{i=1}^n\pi_k(\bsx_i;\bsw^{[s]})}\right).
\end{equation}The function $G^{[q]}(w_{kj}|\bsw^{[s]})$ is a concave function. Moreover, it is a univariate function w.r.t $w_{kj}$. 
We can therefore maximize it globally and w.r.t each coeffcient $w_{kj}$ separately and thus avoid matrix inversion. 
Indeed, let us denote by
$$F^{[q]}_{kjm}(w_{kj}) =  \sum\limits_{i=1}^n \tau_{ik}^{[q]}x_{ij}w_{kj}- \sum\limits_{i=1}^n \frac{\pi_k(\bsx_i;\bsw^{[s]})}{p+1}e^{(p+1)x_{ij}(w_{kj}- w_{kj}^{[s]})}-\frac{\rho}{2}w_{kj}^2+B_{kj}(\bsw^{[s]}),$$
hence, $G^{[q]}(w_{kj}|\bsw^{[s]})$ can be rewritten as
$$G^{[q]}(w_{kj}|\bsw^{[s]}) = 
\begin{cases}
F^{[q]}_{kjm}(w_{kj}) - \gamma_kw_{kj} &,\text{ if } w_{kj} > 0\\
F^{[q]}_{kjm}(0) &,\text{ if } w_{kj} = 0\\
F^{[q]}_{kjm}(w_{kj}) + \gamma_kw_{kj} &,\text{ if } w_{kj} < 0
\end{cases}.$$
We therefore have both $F^{[q]}_{kjm}(w_{kj}) - \gamma_kw_{kj}$ and $F^{[q]}_{kjm}(w_{kj}) + \gamma_kw_{kj}$ are smooth concave functions. 
Thus, one can use one-dimensional Newton-Raphson algorithm to find the global maximizers of these functions and compare with $F^{[q]}_{kjm}(0)$ in order to update $w^{[s]}_{kj}$ by $$w^{[s+1]}_{kj} = \arg\max_{w_{kj}}G^{[q]}(w_{kj}|\bsw^{[s]}).$$
The update of $w_{kj}$ can then be computed by a one-dimensional generalized Newton-Raphson (NR) algorithm, which updates, after starting from and initial value $w_{kj}^{[0]} = w_{kj}^{[s]}$, at each iteration $t$ of the NR, according to the following updating rule: 
$$w_{kj}^{[t+1]} = w_{kj}^{[t]} - \Bigl(\frac{\partial^2G^{[q]}(w_{kj}|\bsw^{[s]})}{\partial^2w_{kj}}\Bigl)^{-1}\Bigl|_{w_{kj}^{[t]}}\frac{\partial G^{[q]}(w_{kj}|\bsw^{[s]})}{\partial w_{kj}}\Bigl|_{w_{kj}^{[t]}} ,$$
where the first and the scalar gradient and hessian are respectively given by:
$$\frac{\partial G^{[q]}(w_{kj}|\bsw^{[s]})}{\partial w_{kj}} = \begin{cases}
U(w_{kj}) - \gamma_k &, G^{[q]}(w_{kj}|\bsw^{[s]})  = F^{[q]}_{kjm}(w_{kj}) - \gamma_kw_{kj}\\
U(w_{kj}) + \gamma_k &, G^{[q]}(w_{kj}|\bsw^{[s]})  = F^{[q]}_{kjm}(w_{kj}) + \gamma_kw_{kj}
\end{cases},$$
and
$$\frac{\partial^2 G^{[q]}(w_{kj}|\bsw^{[s]})}{\partial^2w_{kj}} = -(p+1) \sum\limits_{i=1}^n x_{ij}^2\pi_k(\bsx_i;\bsw^{[s]}) e^{(p+1)x_{ij}(w_{kj}- w_{kj}^{[s]})} -\rho,$$
with
$$U(w_{kj}) =\sum\limits_{i=1}^n \tau_{ik}^{[q]}x_{ij} - \sum\limits_{i=1}^n x_{ij}\pi_k(\bsx_i;\bsw^{[s]}) e^{(p+1)x_{ij}(w_{kj}- w_{kj}^{[s]})} - \rho w_{kj}.$$
Unluckily, while this method allows to compute separate univariate updates by globally maximizing concave functions, it has some drawbacks. First, we found the same behaviour of the MM algorithm for this non-smooth function setting as in \cite{Hun05}: once a coefficient is set to be zero, it may never reenter the model at a later stage of the algorithm. Second, the MM algorithm can stuck on non-optimal points of the objective function.
 \cite{Schi10} made an interesting study on the convergence of the MM algorithms for nonsmoothly penalized objective functions, in which they proof that with some conditions on the minorizing function (see Theorem 2.1 of \cite{Schi10}), then the MM algorithm will converge to the optimal value. One of these conditions requires the minorizing function must be strickly positive, which is not guaranteed in our method, since we use the arithmetic-geometric mean inequality in (\ref{AGinequality}) to construct our surrogate function. 
 Hence, we just ensure that the value of $Q(\bsw;\bstheta^{[q]})$ will not decrease in our algorithm. In the next section, we propose updating $(w_{k0}, \bsw_k)$ by using coordinate ascent algorithm. This approach overcomes this weakness of the MM algorithm.
\subsubsection{Coordinate ascent algorithm for updating the gating network}
\label{sec::CAalgorithm}
We now consider another approach for updating $(w_{k0}, \bsw_k)$ by using coordinate ascent algorithm. Indeed, based on \cite{Tse88, Tse01}, with  regularity  conditions, then the coordinate ascent algorithm is successful in updating $\bsw$. 
Thus, the $\bsw$ parameters are updated in a cyclic way, where a coefficient $w_{kj}$ ($j \not= 0$) is updated at each time, while fixing the other parameters to their previous values. Hence, at each iteration one just needs to update only one parameter. With this setting, the update of $w_{kj}$ is performed by maximizing the component $(k,j)$ of (\ref{eq.Qw}) given by
\begin{equation}
Q(w_{kj};\bstheta^{[q]}) = F(w_{kj};\bstheta^{[q]}) -\gamma_k|w_{kj}|,
\end{equation}
where
\begin{equation}
F(w_{kj};\bstheta^{[q]}) = \sum_{i=1}^n\tau_{ik}^{[q]}(w_{k0}+\bsw^T_k \bsx_i) - \sum_{i=1}^n\log\Bigl[1+\sum_{l=1}^{K-1}e^{w_{l0}+\bsw^T_l \bsx_i}\Bigl] - \frac{\rho}{2}w_{kj}^2.
\end{equation}
Hence, $Q(w_{kj};\bstheta^{[q]})$ can be rewritten as
$$Q(w_{kj};\bstheta^{[q]}) = 
\begin{cases}
F(w_{kj};\bstheta^{[q]}) -\gamma_k w_{kj} &,\text{ if }w_{kj} > 0\\
F(0;\bstheta^{[q]}) &,\text{ if }w_{kj} = 0\\
F(w_{kj};\bstheta^{[q]}) +\gamma_k w_{kj} &,\text{ if }w_{kj} < 0
\end{cases}.$$
Again, both $F(w_{kj};\bstheta^{[q]}) -\gamma_k w_{kj}$ and $F(w_{kj};\bstheta^{[q]}) +\gamma_k w_{kj}$ are smooth concave functions. Thus, one can use one-dimensional generalized Newton-Raphson algorithm with initial value $w_{kj}^{[0]} = w_{kj}^{[q]}$  to find the maximizers of these functions and compare with $F(0;\bstheta^{[q]}) $ in order to update $w^{[s]}_{kj}$ by $$w^{[s+1]}_{kj} = \arg\max_{w_{kj}}Q(w_{kj};\bstheta^{[q]}),$$
where $s$ denotes the $s$th loop of the coordinate ascent algorithm. The update of $w_{kj}$ is therefore computed iteratively after starting from and initial value $w_{kj}^{[0]} = w_{kj}^{[s]}$ following the update equation
{\begin{equation}
w_{kj}^{[t+1]} = w_{kj}^{[t]} - \Bigl(\frac{\partial^2 Q(w_{kj};\bstheta^{[q]})}{\partial^2w_{kj}}\Bigl)^{-1}\Bigl|_{w_{kj}^{[t]}}\frac{\partial Q(w_{kj};\bstheta^{[q]})}{\partial w_{kj}}\Bigl|_{w_{kj}^{[t]}} ,
\end{equation}}where $t$ in the inner NR iteration number, and the one-dimensional gradient and hessian functions are respectively given by
\begin{equation}\frac{\partial Q(w_{kj};\bstheta^{[q]})}{\partial w_{kj}} = \begin{cases}
U(w_{kj}) - \gamma_k &,\text{ if }Q(w_{kj};\bstheta^{[q]})  = F(w_{kj};\bstheta^{[q]}) -\gamma_k w_{kj}\\
U(w_{kj}) + \gamma_k &,\text{ if }Q(w_{kj};\bstheta^{[q]})  = F(w_{kj};\bstheta^{[q]}) +\gamma_k w_{kj}
\end{cases},\end{equation}
and
\begin{equation*}
\!\!\!\!\frac{\partial^2 Q(w_{kj};\bstheta^{[q]})}{\partial^2 w_{kj}} = -\sum_{i=1}^n \frac{x_{ij}^2e^{w_k0+x_i^T\bsw_k}(C_i(w_{kj})-e^{w_{k0}+x_i^T\bsw_k})}{C_i^2(w_{kj})}-\rho.
\end{equation*}with 
$$U(w_{kj}) = \sum_{i=1}^n x_{ij}\tau_{ik}^{[q]} - \sum_{i=1}^n \frac{x_{ij}e^{w_{k0}+\bsx_i^T\bsw_k}}{C_i(w_{kj})} - \rho w_{kj},$$
and
$$C_i(w_{kj}) = 1+ \sum_{l\not=k}e^{w_{l0}+x_i^T\bsw_l}+e^{w_{k0}+x_i^T\bsw_k},$$
is a univariate function of $w_{kj}$ when fixing other parameters. 
For other parameter we set $w_{lh}^{[s+1]} = w_{lh}^{[s]}$.\\
Similarly, for the update of $w_{k0}$, a univariate Newton-Raphson algorithm with initial value $w_{k0}^{[0]} = w_{k0}^{[q]}$ can be used to provide the update $w^{[s]}_{k0}$ given by 
$$w^{[s+1]}_{k0} = \arg\max_{w_{k0}}Q(w_{k0};\bstheta^{[q]}),$$
where $Q(w_{k0};\bstheta^{[q]})$ is a univariate concave function given by 
{\begin{equation}
Q(w_{k0};\bstheta^{[q]}) = \sum_{i=1}^n\tau_{ik}^{[q]}(w_{k0}+\bsx_i^T \bsw_k) - \sum_{i=1}^n\log\Bigl[1+\sum_{l=1}^{K-1}e^{w_{l0}+\bsx_i^T \bsw_l}\Bigl],
\end{equation}}with
{\begin{equation}\frac{\partial Q(w_{k0};\bstheta^{[q]})}{\partial w_{k0}} = \sum_{i=1}^n\tau_{ik}^{[q]} - \sum_{i=1}^n\frac{e^{w_{k0}+\bsx_i^T\bsw_k}}{C_i(w_{k0})}
\end{equation}}
and
{\begin{equation}\frac{\partial^2 Q(w_{k0};\bstheta^{[q]})}{\partial^2 w_{k0}} = -\sum_{i=1}^n \frac{e^{w_{k0}+x_i^T\bsw_k}(C_i(w_{k0})-e^{w_{k0}+x_i^T\bsw_k})}{C_i^2(w_{k0})}.
\end{equation}}The other parameters are fixed while updating $w_{k0}$. 
By using the coordinate ascent algorithm, we have univariate updates, and make sure that the parameters $w_{kj}$ may change during the algorithm even after they shrink to zero at an earlier stage of the algorithm. 
\subsubsection{Updating the experts network}
Now once we have these two methods to update the gating network parameters, we move on updating the experts network parameters $(\{\bsbeta,\sigma^2\})$. To do that, we first perform the update for $(\beta_{k0}, {\bf \beta}_k)$, while fixing $\sigma_k$. This corresponds to solving $K$ separated weighted Lasso problems. 
Hence, we choose to use a coordinate ascent algorithm to deal with this. 
Actually, in this situation the coordinate ascent algorithm can be seen as a special case of the MM algorithm, and hence, this updating step is common to both of the proposed algorithms. 
More specifically, the update of $\beta_{kj}$ is performed by maximizing 
\begin{equation}
Q(\bsbeta, \sigma;\bstheta^{[q]}) = \sum_{i=1}^n\sum_{k=1}^K\tau_{ik}^{[q]}\log\mathcal{N}(y_i; \beta_{k0} +\bsbeta^T_k \bsx_i,\sigma_k^2) - \sum_{k=1}^K\lambda_k\|\bsbeta_k\|_1;
\end{equation}using a coordinate ascent algorithm, with initial values $(\beta_{k0}^{[0]}, \bsbeta_k^{[0]}) = (\beta_{k0}^{[q]}, \bsbeta_k^{[q]})$. We obtain closed-form coordinate updates that can be computed for each component following the results in \cite[sec. 5.4]{TH15c}, and are given by
\begin{equation}\label{CDb}
\beta_{kj}^{[s+1]} = \frac{{\bf\mathcal{S}}_{\lambda_k\sigma_k^{(s)2}}\bigl(\sum_{i=1}^n\tau_{ik}^{[q]}r_{ikj}^{[s]}x_{ij}\bigl)}{\sum_{i=1}^n\tau_{ik}^{[q]}x_{ij}^2},
\end{equation}with $r_{ikj}^{[s]} = y_i-\beta_{k0}^{[s]}- \bsbeta^{[s]T}_k \bsx_i + \beta_{kj}^{[s]}x_{ij}$ and ${\bf\mathcal{S}}_{\lambda_k\sigma_k^{(s)2}}(.)$ is a  soft-thresholding operator defined by {$[{\bf\mathcal{S}}_\gamma(u)]_j = \text{sign}(u_j)(|u_j|-\gamma)_+$} and {$(x)_+$}  a shorthand for $\max\{x, 0\}$. For  $h\not=j$, we set $\beta_{kh}^{[s+1]} = \beta_{kh}^{[s]}$. At each iteration $m$, $\beta_{k0}$ is updated by
\begin{equation}
\label{Bk0}\beta_{k0}^{[s+1]} = \frac{\sum_{i=1}^n\tau_{ik}^{[q]}(y_i - \bsbeta^{[s+1]T}_k\bsx_i)}{\sum_{i=1}^n\tau_{ik}^{[q]}}\cdot
\end{equation}
In the next step, we take {$(w_{k0}^{[q+2]}, \bsw_k^{[q+2]}) = (w_{k0}^{[q+1]}, \bsw_k^{[q+1]})$}, {$(\beta_{k0}^{[q+2]}, \bsbeta_k^{[q+2]}) = (\beta_{k0}^{[q+1]}, \bsbeta_k^{[q+1]})$}, rerun the E-step, and update $\sigma_k^{2}$ according to the standard update of a weighted Gaussian regression
\begin{equation}\label{SHe}\displaystyle\sigma_k^{2[q+2]} = \frac{\sum_{i=1}^n\tau_{ik}^{[q+1]}(y_i - \beta_{k0}^{[q+2]} - {\bsbeta^{[q+2]}_k}^T \bsx_i)^2}{\sum_{i=1}^n\tau_{ik}^{[q+1]}}\cdot
\end{equation}
Each of the two proposed algorithms is iterated until the change in $PL(\bstheta)$ is small enough.
These algorithms increase the penalised log-likelihood function (\ref{eq:PenLoglik MoE}) as shown in Appendix.
 Also we can directly get zero coefficients without any thresholding unlike in \cite{Kha10, Hun05}.

The R codes of the developed algorithms and the documentation are publicly available on this link\footnote{{https://chamroukhi.users.lmno.cnrs.fr/software/RMoE/RCode-RMoE.zip}}. An R package will be submitted and available soon on the CRAN.
\subsection{Algorithm tuning and model selection}

In practice, appropriate values of the tuning parameters $(\lambda, \gamma,\rho)$ should be chosen. To select the tuning parameters, we propose a modified BIC with a grid search scheme, 
as an extension of the criterion used in \cite{Sta10} for regularized mixture of regressions. First, assume that $K_0\in\{K_1,\hdots,K_M\}$ whereupon $K_0$ is the true number of expert components. For each value of $K$, we choose a grid of the tuning parameters. 
Consider grids of values $\{\lambda_1,\hdots,\lambda_{M_1}\}$, $\{\gamma_1,\hdots,\gamma_{M_2}\}$ in the size of $\sqrt{n}$ and a small enough value of $\rho\approx O(\log n)$ for the ridge turning parameter.  $\rho = 0.1\log n$ can be used in practice. For a given triplet $(K, \lambda_i,\gamma_j)$, we select the maximal penalized log-likelihood estimators $\widehat{\bstheta}_{K, \lambda, \gamma}$ using each of our hybrid EM algorithms presented above. Then, the following modified BIC criterion,  
\begin{equation}\label{BIC}
\text{BIC}(K, \lambda,\gamma) = L(\widehat{\bstheta}_{K, \lambda, \gamma}) - DF(\lambda, \gamma)\frac{\log n}{2},
\end{equation}where $DF(\lambda, \gamma)$ is the estimated number of non-zero coefficients in the model,  is computed. Finally, the model with parameters $(K, \lambda, \gamma) = (\tilde{K}, \tilde{\lambda}, \tilde{\gamma})$ which maximizes the modified BIC value, is selected. 
While the problem of choosing optimal values of the tuning parameters for penalized MoE models is still an open research, the modified BIC performs reasonably well in our experiments.   
 
%
%
%

\section{Experimental study}
\label{sec: Experiments}
We study the performance of our methods on both simulated data and real data. We compare the results of our two algorithms (Lasso+$\ell_2$ (MM) and Lasso+$\ell_2$ with coordinate ascent (CA)), with the following four methods: {\it i)} the standard non-penalized MoE (MoE), {\it ii)} the MoE with $\ell_2$ regularization (MoE+$\ell_2$), {\it iii)} the mixture of linear regressions with Lasso penalty (MIXLASSO), and the {\it iv)} MoE with BIC penalty for feature selection. 
We consider several evaluation criteria to assess the performance of the models, including sparsity, parameters estimation and clustering criteria.

\subsection{Evaluation criteria}
We compare the results of all the models for three different criteria: sensitivity/specificity, parameters estimation, and clustering performance for simulation data. The sensitivity/specificity is defined by
\begin{itemize}
\item {\it Sensitivity:} proportion of correctly estimated zero coefficients;
\item {\it Specificity:} proportion of correctly estimated nonzero coefficients.
\end{itemize}
In this way, we compute the ratio of the estimated zero/nonzero coefficients to the true number of zero/nonzero coefficients of the true parameter for each component. 
In our simulation, the proportion of correctly estimated zero coefficients and nonzero coefficients have been calculated for each data set  for the experts parameters and the gating parameters, and we present the average proportion of these criteria computed over 100 different data sets. 
Also, to deal with the label switching before calculating these criteria, we permuted the estimated coefficients based on an ordered between the expert parameters. If the label switching happens, one can permute the expert parameters and the gating parameters then replace the second one $\bsw_k^{per}$ with $\bsw_k^{per} - \bsw_K^{per}$. By doing so, we ensure that the log-likelihood will not change, that means $L(\hat{\bstheta}) = L(\hat{\bstheta}^{per})$ and these parameters satisfy the initialized condition $\bsw_K^{per} = \boldsymbol{0}$. However, the penalized log-likelihood value can be different from the one before permutation. So this may result in misleading values of the sparsity criterion of the model when we permute the parameters. 
However, for $K = 2$ both log-likelihood function and the penalized log-likelihood function will not change since we have $\bsw_1^{per} = -\bsw_1$.\\
For the second criterion of parameter estimation, we compute the mean and standard deviation of both penalized parameters and non penalized parameters in comparison with the true value $\bstheta$. We also consider the mean squared error (MSE) between each component of the true parameter vector and the estimated one, which is given by $\|\theta_j - \hat\theta_j\|^2$. \\
For the clustering criterion, once the parameters are estimated and permuted, the provided conditional component probabilities $\hat{\tau}_{ik}$ defined in (\ref{eq:RMoE post prob}) represent a soft partition of the data. A hard partition of the data is given by applying the Bayes's allocation rule
$$\hat{z}_i = \arg\max_{k =1}^K \tau_{ik}(\widehat{\bstheta}),$$ 
where $\hat{z}_i$ represents the estimated cluster label for the $i$th observation. 
Given the estimated and true cluster labels, we therefore compute the  correct classification rate and the Adjusted Rand Index (ARI). \\
Also, we note that for the standard MoE with BIC penalty, we consider a pool of $5\times 4 \times 5 = 100$ submodels. Our EM algorithm with coordinate ascent has been used with zero penalty coefficients and without updating the given zero parameters in the experts and the gating network to obtain the (local) MLE of each submodel. 
After that, the BIC criterion in (\ref{BIC}) was used to choose the best submodel among 100 model candidates.

\subsection{Simulation study}
For each data set, we consider $n=300$ predictors $\bsx$  generated from a multivariate Gaussian distribution with zero mean and correlation defined by $\text{corr}(x_{ij}, x_{ij\prime}) = 0.5^{|j-j\prime|}$. The response $Y|\bsx$ is generated from a normal MoE model of $K = 2$ expert components as defined by (\ref{eq.Z generation in MoE}) and (\ref{eq.NMoE regression generative model}), with the following regression coefficients:
\begin{align*}
(\beta_{10},\bsbeta_1)^T &= (0, 0, 1.5, 0, 0, 0, 1)^T;\\
(\beta_{20},\bsbeta_2)^T &= (0, 1, -1.5, 0, 0, 2, 0)^T;\\
(w_{10}, \bsw_1)^T &=  (1, 2, 0, 0, -1, 0, 0)^T;
\end{align*}
and $\sigma_1 = \sigma_2 = \sigma = 1$. 100 data sets were generated for this simulation. The results will be presented in the following sections.
\subsubsection{Sensitivity/specificity criteria}
Table \ref{S/S} presents the sensitivity ($S_1$) and specificity ($S_2$) values for the experts $1$ and $2$ and the gates
for each of the considered models. As it can be seen in the obtained results that the $\ell_2$ and MoE models cannot be considered as model selection methods since their sensitivity almost surely equals zero. However, it is obvious that the Lasso+$\ell_2$, with both the MM and the CA algorithms, performs quite well for experts $1$ and $2$.
The feature selection becomes more difficult for the gate $\pi_k(\bsx;\bsw)$ since there is correlation between features. 
While Lasso+$\ell_2$ using MM (Lasso+$\ell_2$ (MM)) may get trouble in detecting non-zero coefficients in the gating network, the Lasso+$\ell_2$ with coordinate ascent (Lasso+$\ell_2$ (CA)) performs quite well. 
The MIXLASSO, can detect the zero coefficients in the experts but it will be shown in the later clustering results that this model has a poor result when clustering the data. Note that for the MIXLASSO we do not have gates, so variable ``N/A" is mentioned in the results. 
Finally, while the BIC provides the best results in general, it is hard to apply BIC in reality since the number of submodels may be huge. 
\begin{table}[!h]
\centering
\begin{tabular}{|c|c|c|c|c|c|c|}
\hline
Method &\multicolumn{2}{c|}{Expert 1} &\multicolumn{2}{c|}{Expert 2}& \multicolumn{2}{c|}{Gate}\\
\cline{2-7}
  & $S_1$ & $S_2$ & $S_1$ & $S_2$ & $S_1$ & $S_2$\\
\hline
MoE & 0.000 & 1.000 & 0.000 & 1.000 & 0.000 & 1.000\\
MoE+$\ell_2$ & 0.000 & 1.000 & 0.000 & 1.000 & 0.000 & 1.000\\
MoE-BIC & 0.920 & 1.000 & 0.930 & 1.000 & 0.850 & 1.000\\
MIXLASSO & 0.775 & 1.000 & 0.693 & 1.000 & N/A & N/A\\
Lasso+$\ell_2$ (MM) & 0.720 & 1.000 & 0.777 & 1.000 & 0.815 & 0.615\\
Lasso+$\ell_2$ (CA) & 0.700 & 1.000 & 0.803 & 1.000 & 0.853 & 0.945\\
\hline
\end{tabular}
\caption{Sensitivity ($S_1$) and specificity ($S_2$) results. \label{S/S}}
\end{table}
\subsubsection{Parameter estimation}
The boxplots of all estimated parameters are shown in Figures \ref{Beta1}, \ref{Beta2} and \ref{Gate}. 
It turns out that the MoE and MoE+$\ell_2$ could not be considered as model selection methods. Besides that, by adding the $\ell_2$ penalty functions, we can reduce the variance of the parameters in the gate. The BIC, Lasso+$\ell_2$ (MM) and Lasso+$\ell_2$ (CA) provide sparse results for the model, not only in the experts, but also in the gates. However, the Lasso+$\ell_2$ (MM) in this situation forces the nonzero parameter $w_{14}$ toward zero, and this effects the clustering result. The MIXLASSO can also detect zero coefficients in the experts, but since this model does not have a mixture proportions that depend on the inputs, it is least competitive than others. 
\begin{figure*}[!h]
\centering
\begin{tabular}{cc}
\includegraphics[width = 6 cm]{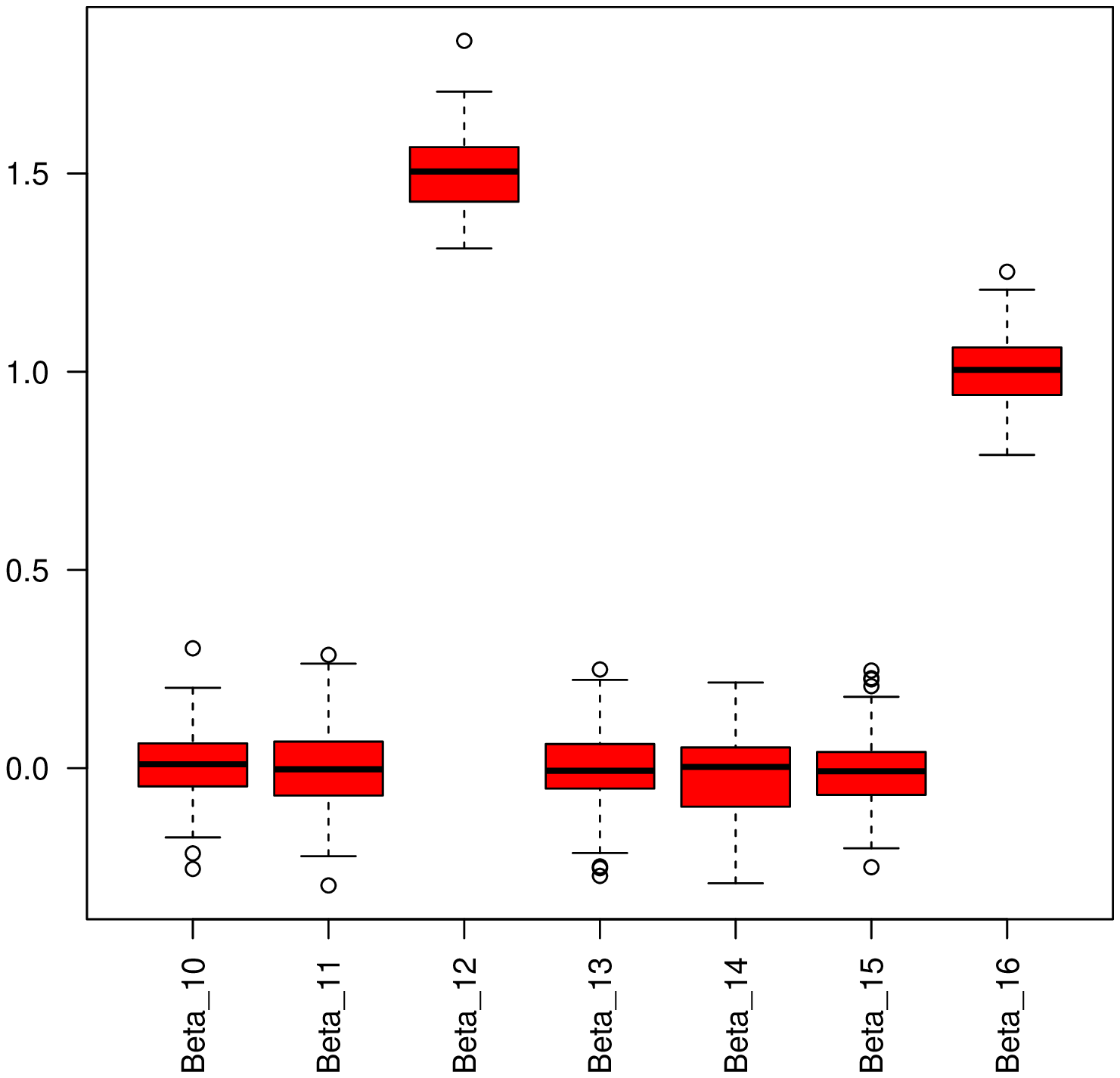}&
\includegraphics[width = 6 cm]{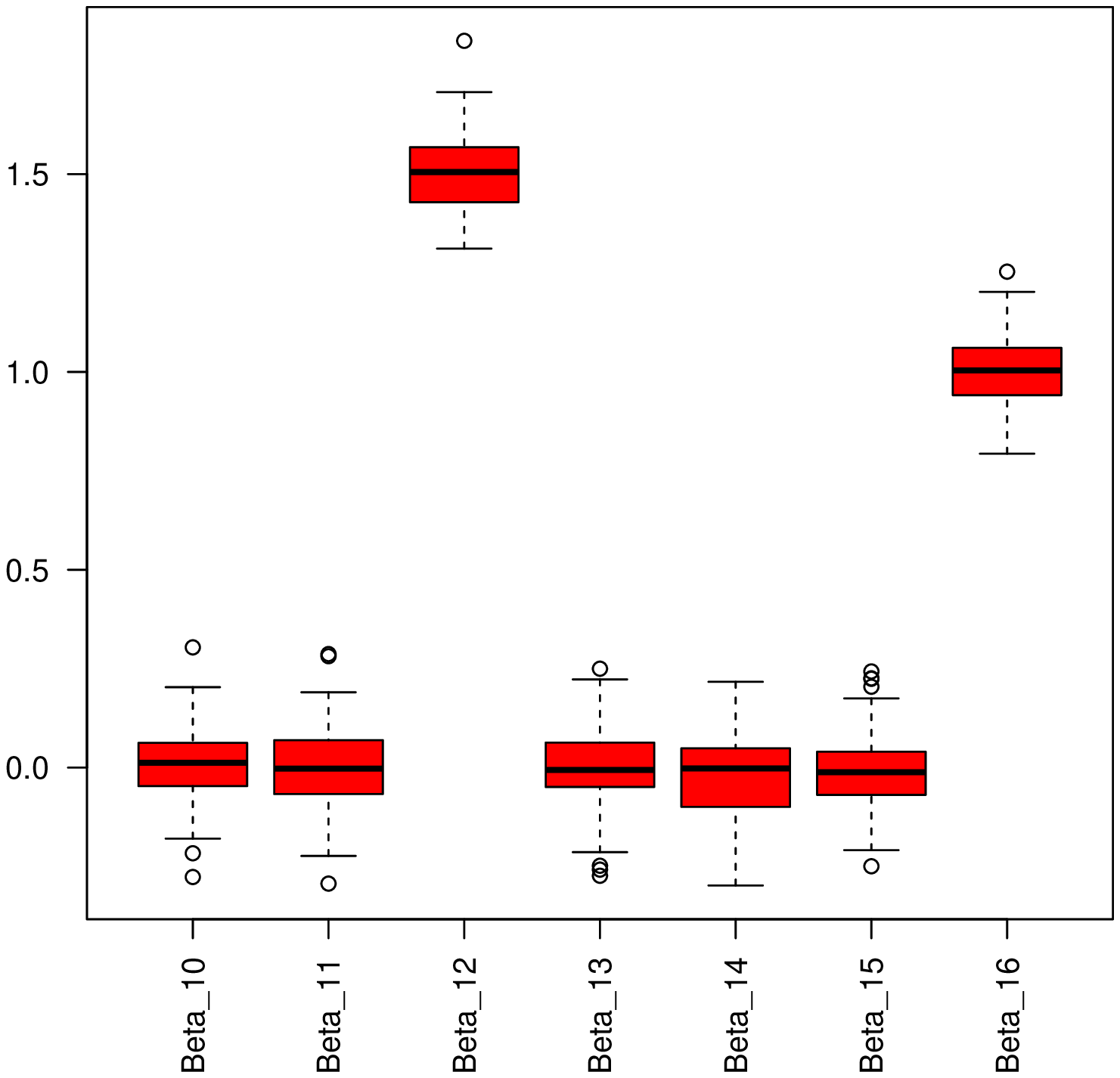}\\
MoE & MoE-$\ell_2$\\  
\includegraphics[width = 6 cm]{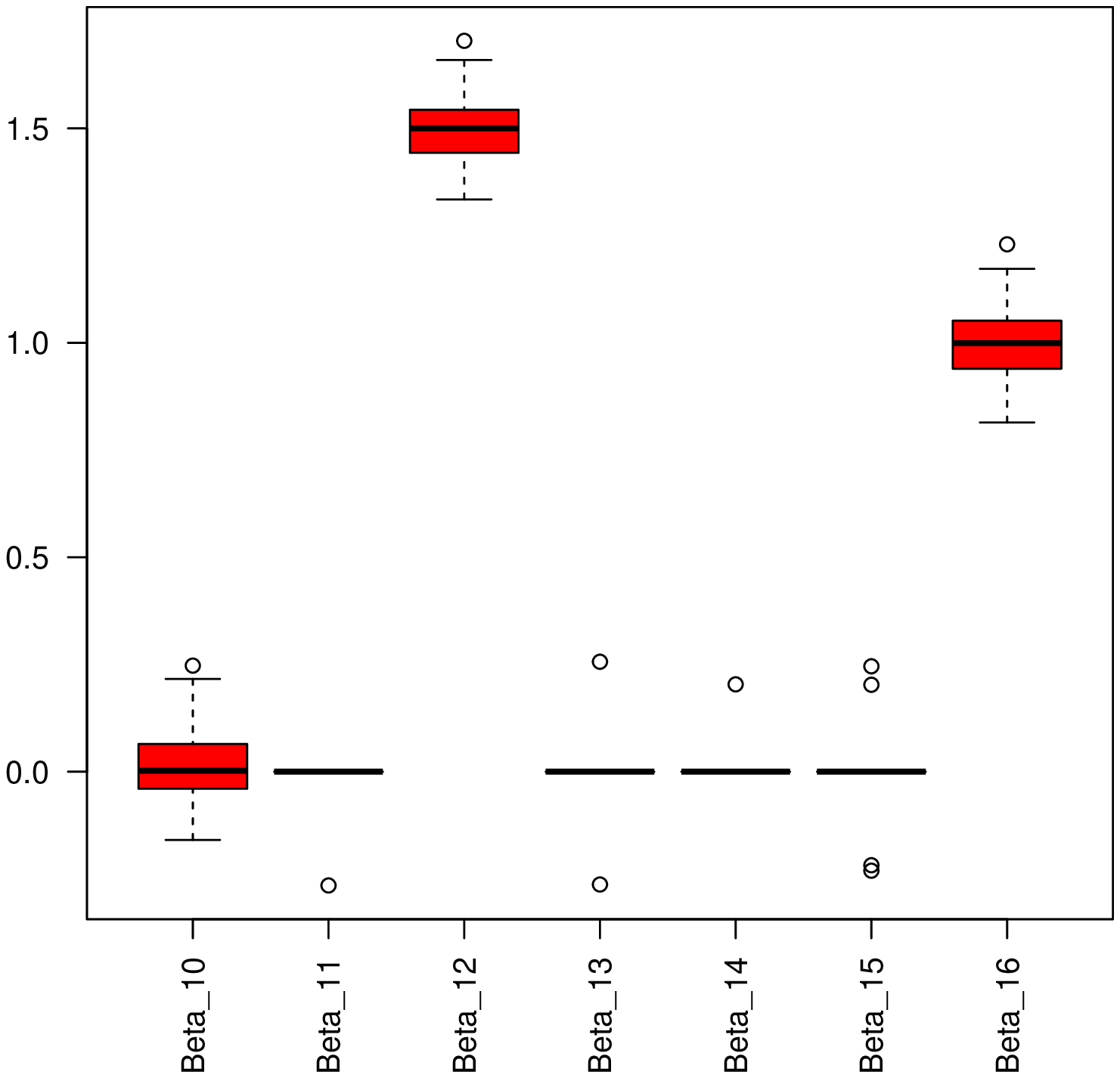}&
\includegraphics[width = 6 cm]{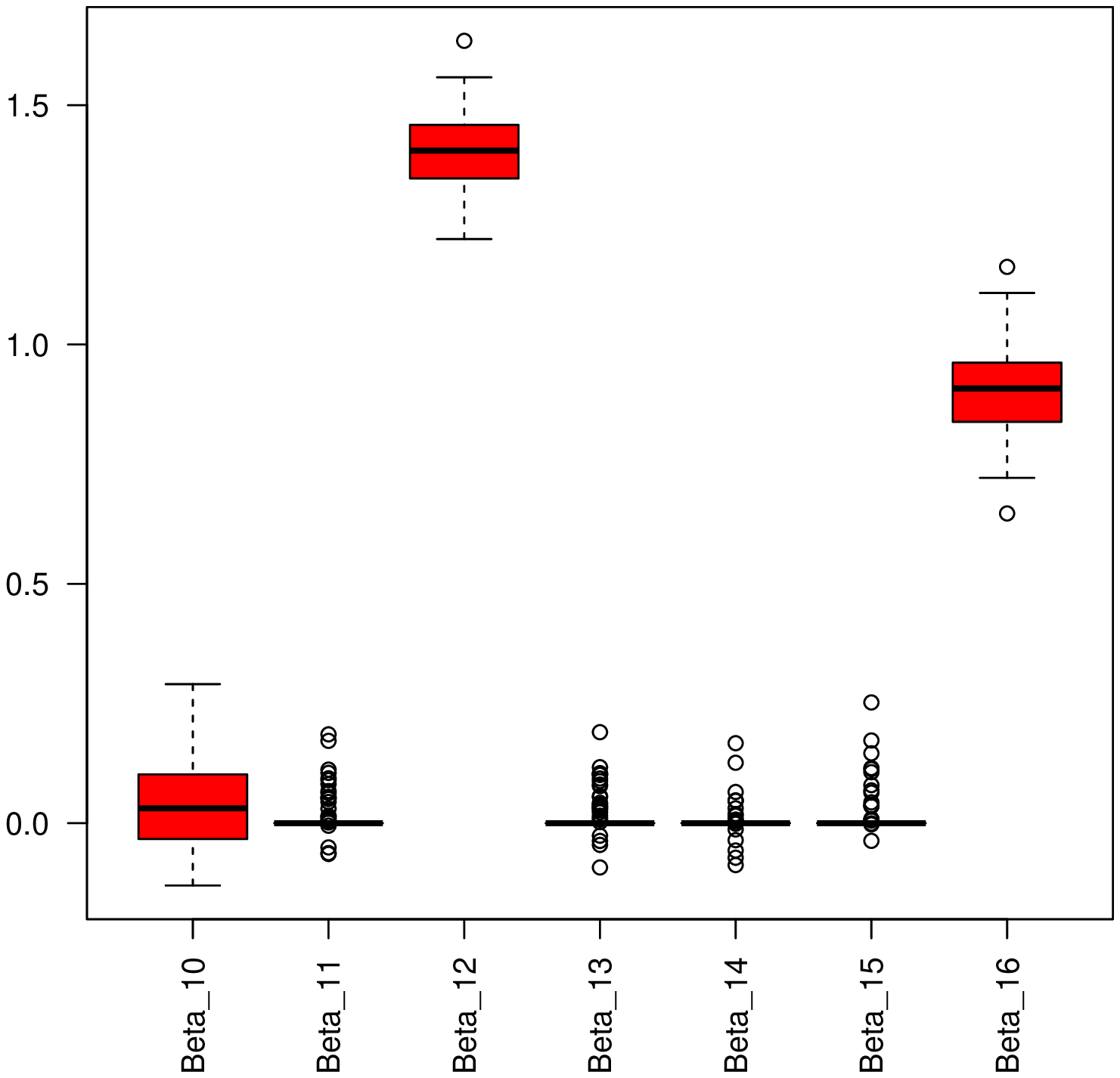}\\
MoE-BIC & MIXLASSO  \\  
\includegraphics[width = 6 cm]{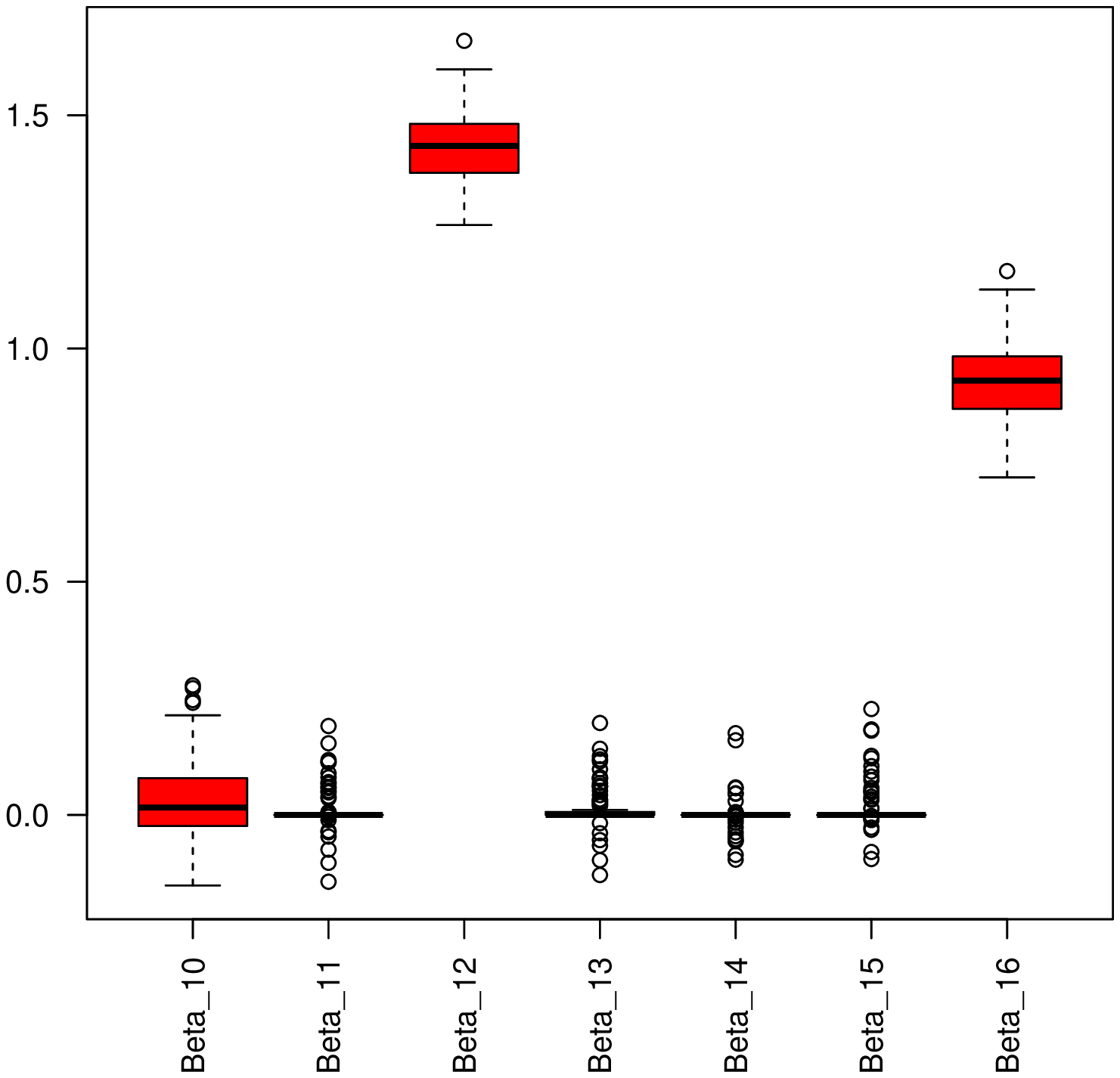}&
\includegraphics[width = 6 cm]{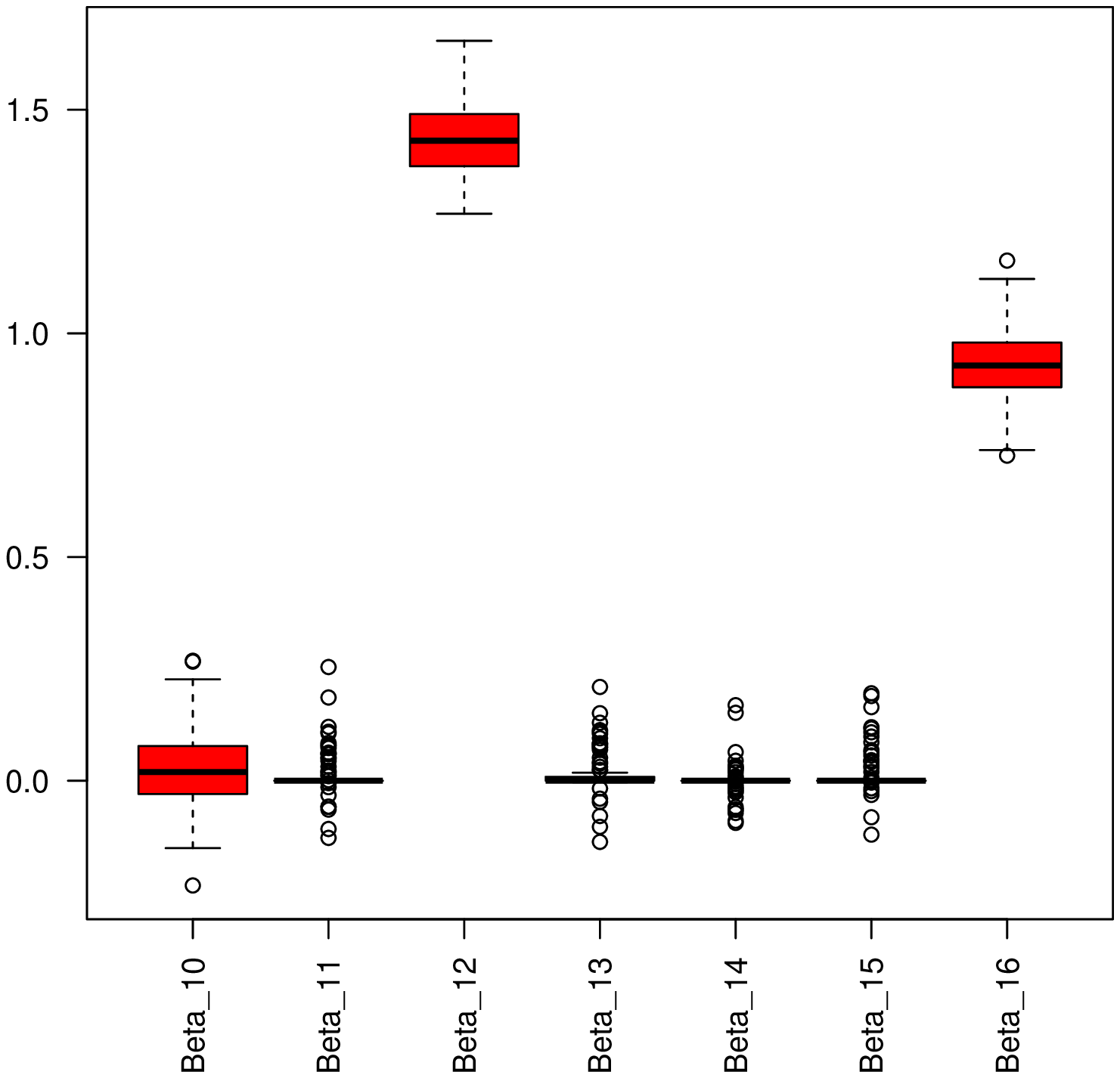}\\
MoE-Lasso + $\ell_2$ (MM) & MoE-Lasso + $\ell_2$ (CA) 
\end{tabular}
\caption{Boxplots of the expert 1's parameter $(\beta_{10},\bsbeta_1)^T = (0,0,1.5,0,0,0,1)^T$.\label{Beta1}}
\end{figure*}

\begin{figure*}[!h]
\centering
\begin{tabular}{cc}
\includegraphics[width = 6 cm]{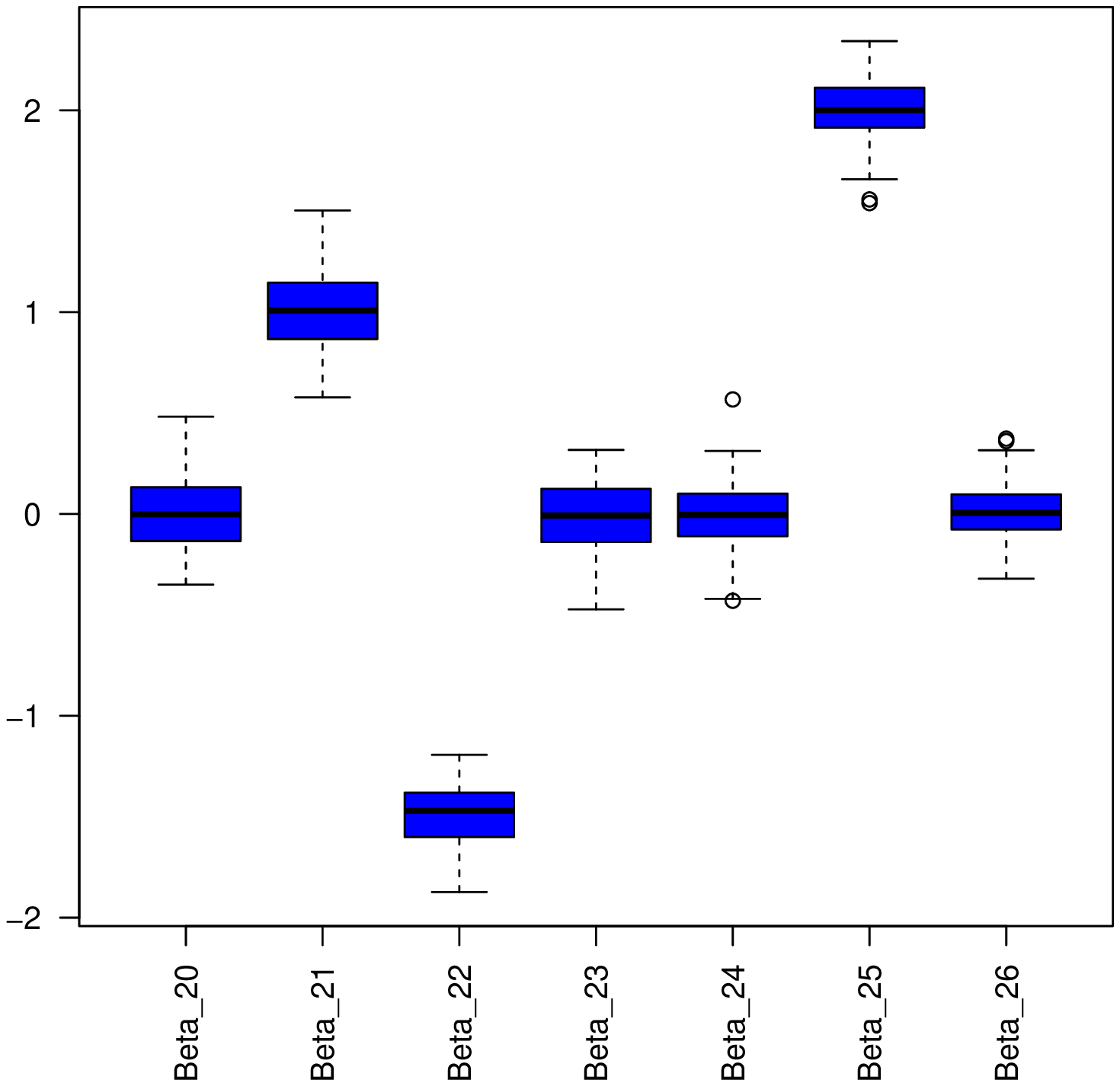}&
\includegraphics[width = 6 cm]{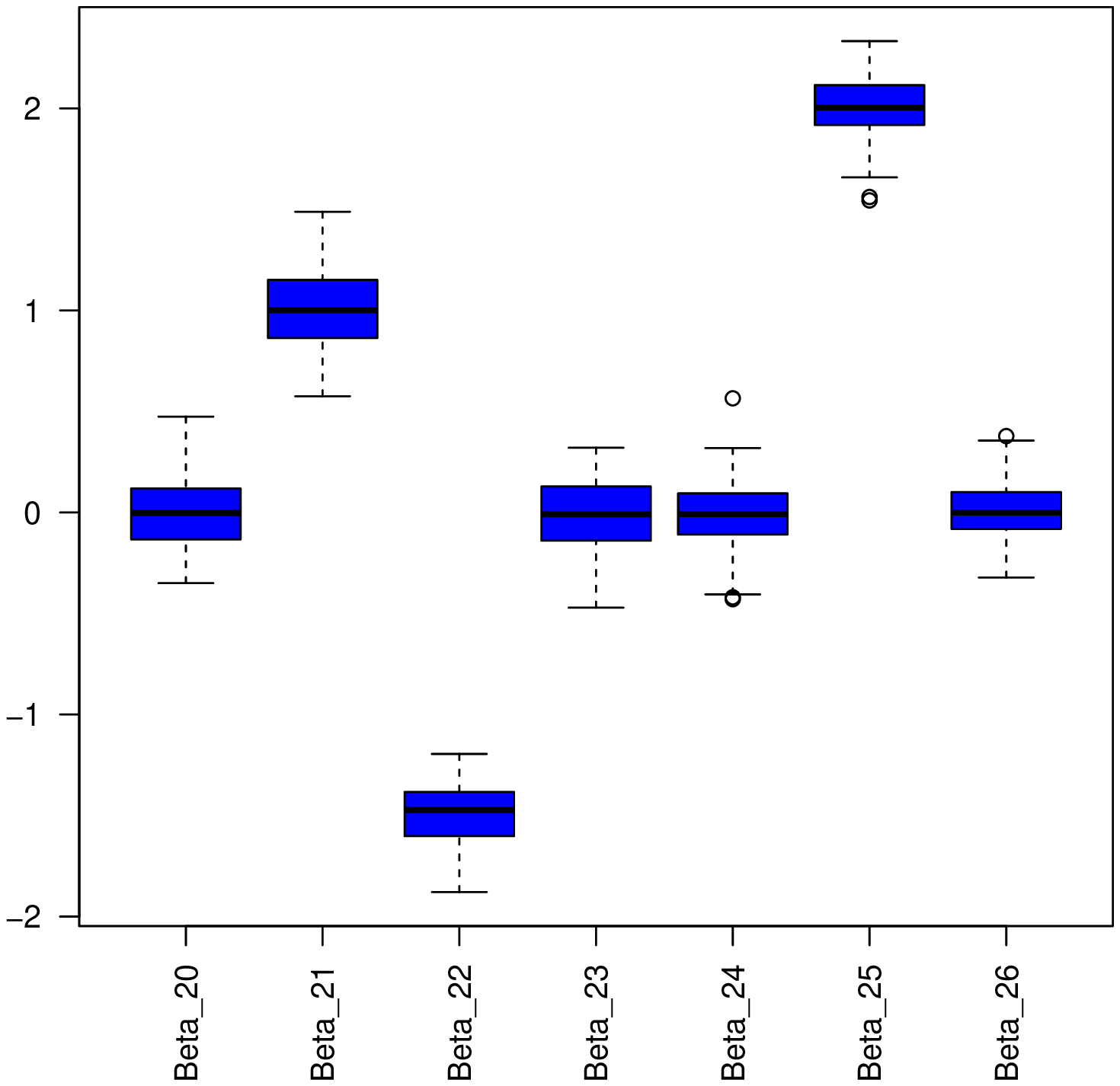} \\
MoE & MoE-$\ell_2$\\  
\includegraphics[width = 6 cm]{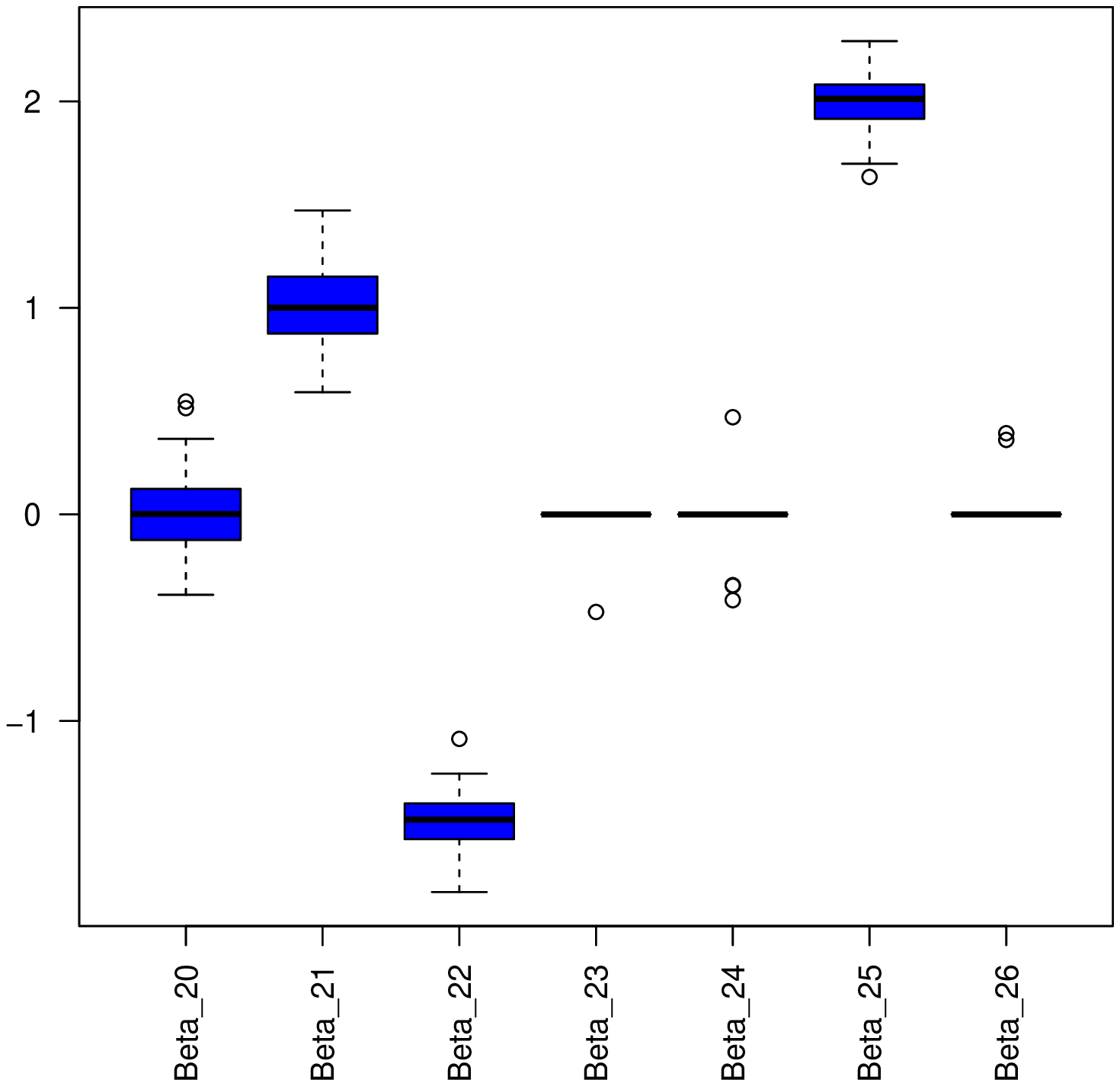} &
\includegraphics[width = 6 cm]{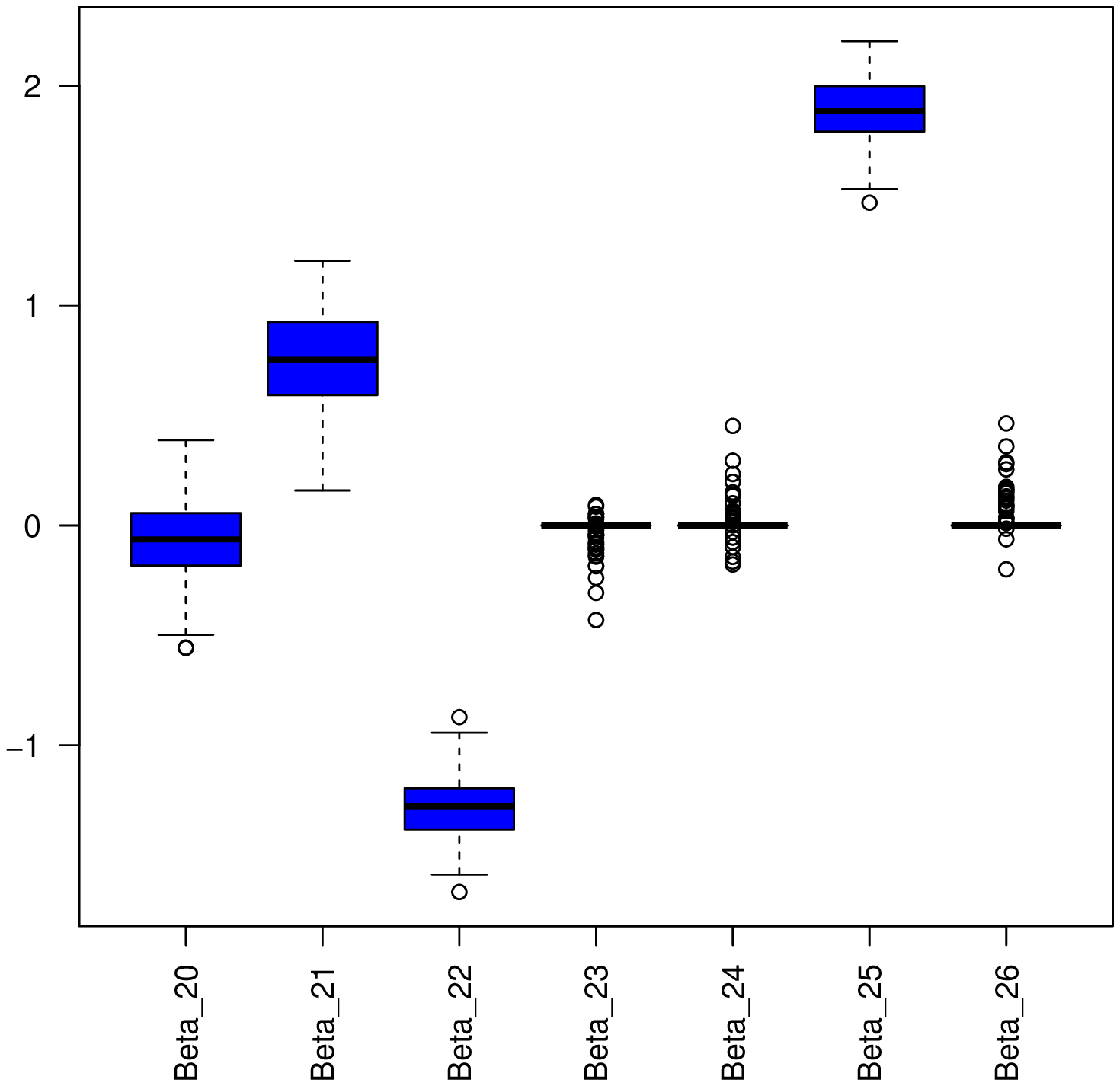}\\
MoE-BIC & MIXLASSO  \\  
\includegraphics[width = 6 cm]{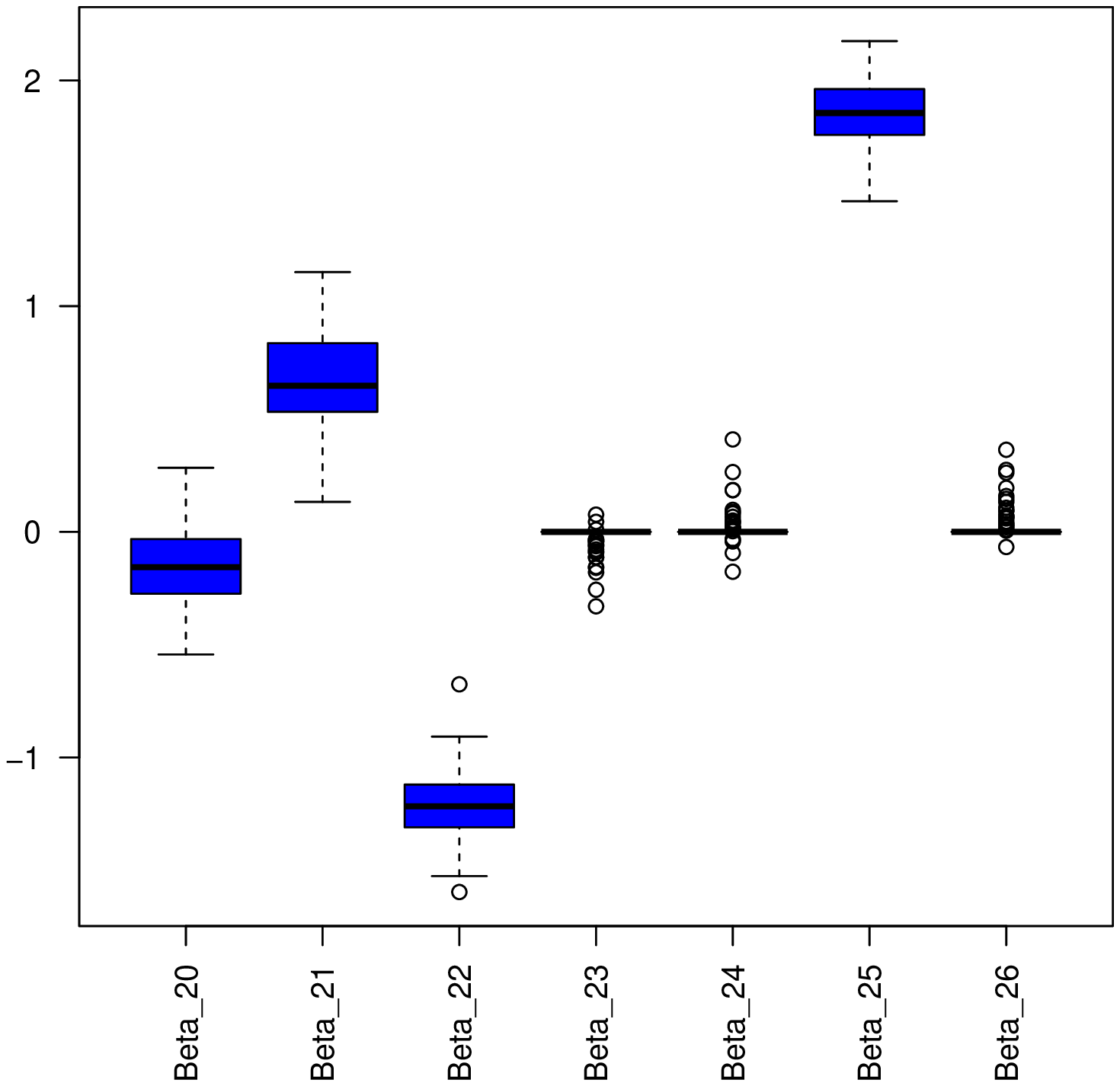}&
\includegraphics[width = 6 cm]{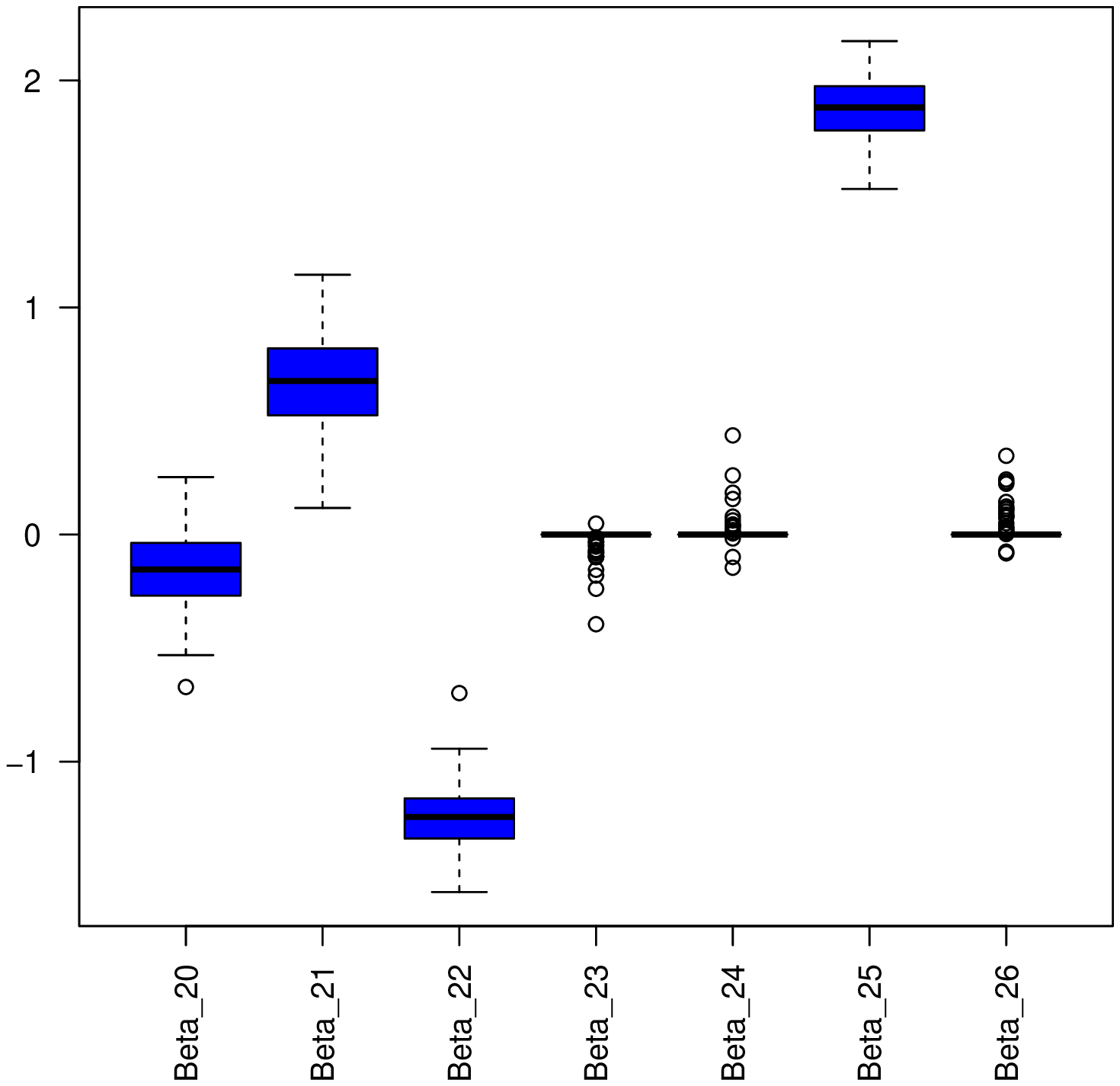}\\
MoE-Lasso + $\ell_2$ (MM) & MoE-Lasso + $\ell_2$ (CA)
\end{tabular}
\caption{Boxplots of the expert 2's parameter $(\beta_{20},\bsbeta_2)^T = (0,1,-1.5,0,0,2,0)^T$.\label{Beta2}}
\end{figure*}

\begin{figure*}[!h]
\centering
\begin{tabular}{cc}
\includegraphics[width = 6 cm]{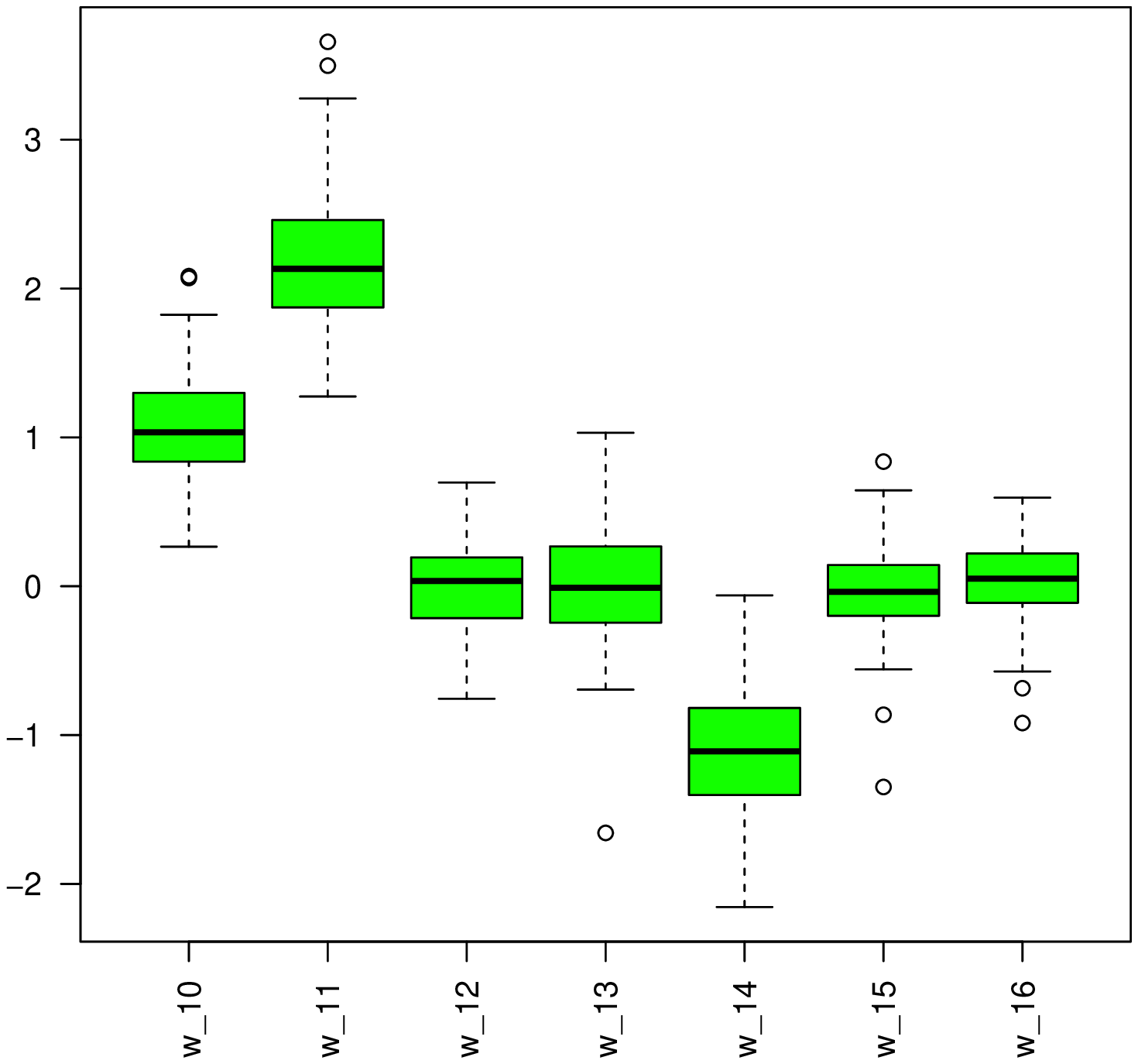} &
\includegraphics[width = 6 cm]{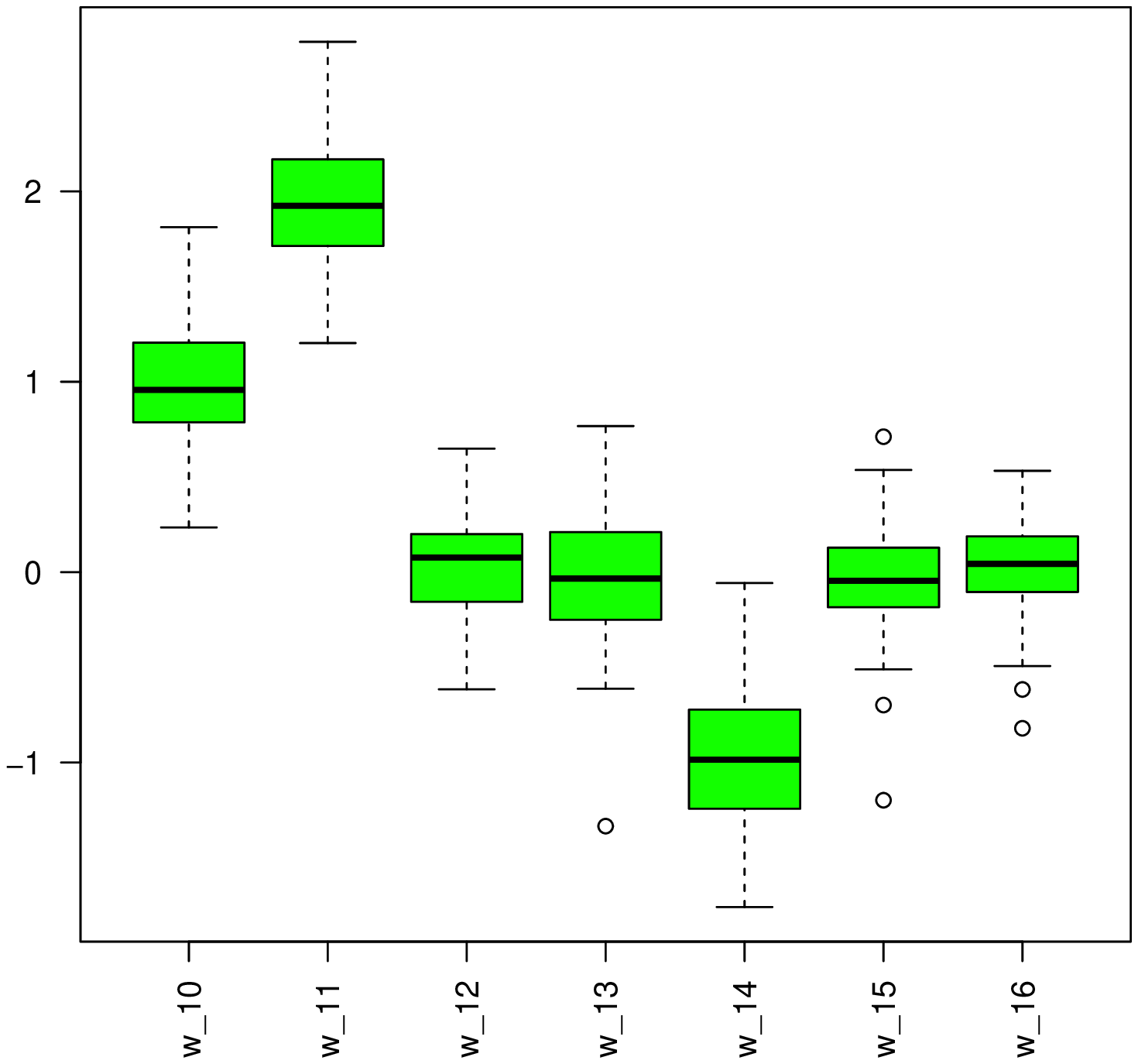} \\
MoE & MoE-$\ell_2$\\  
\includegraphics[width = 6 cm]{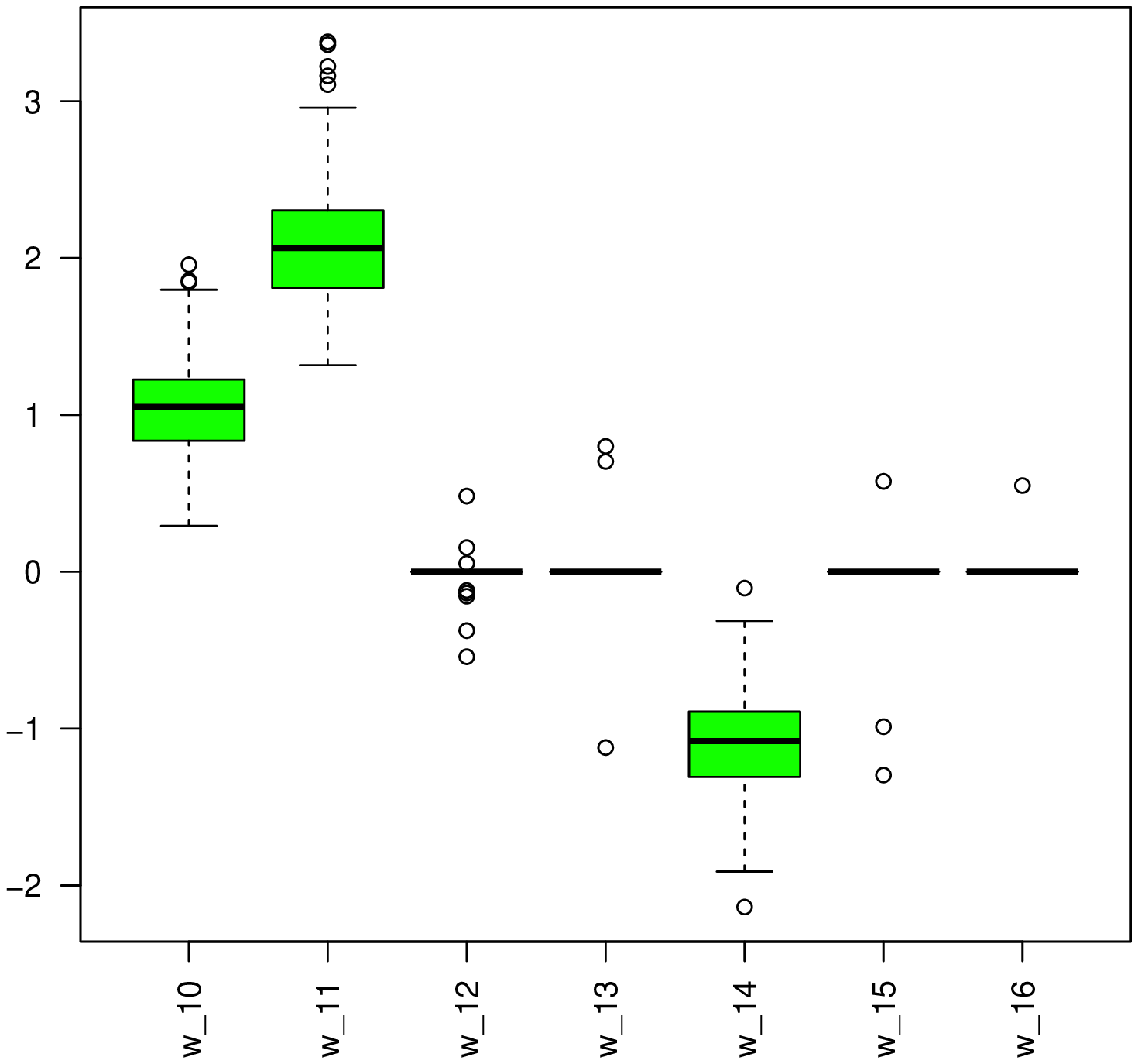}& \\
MoE-BIC &\\
\includegraphics[width = 6 cm]{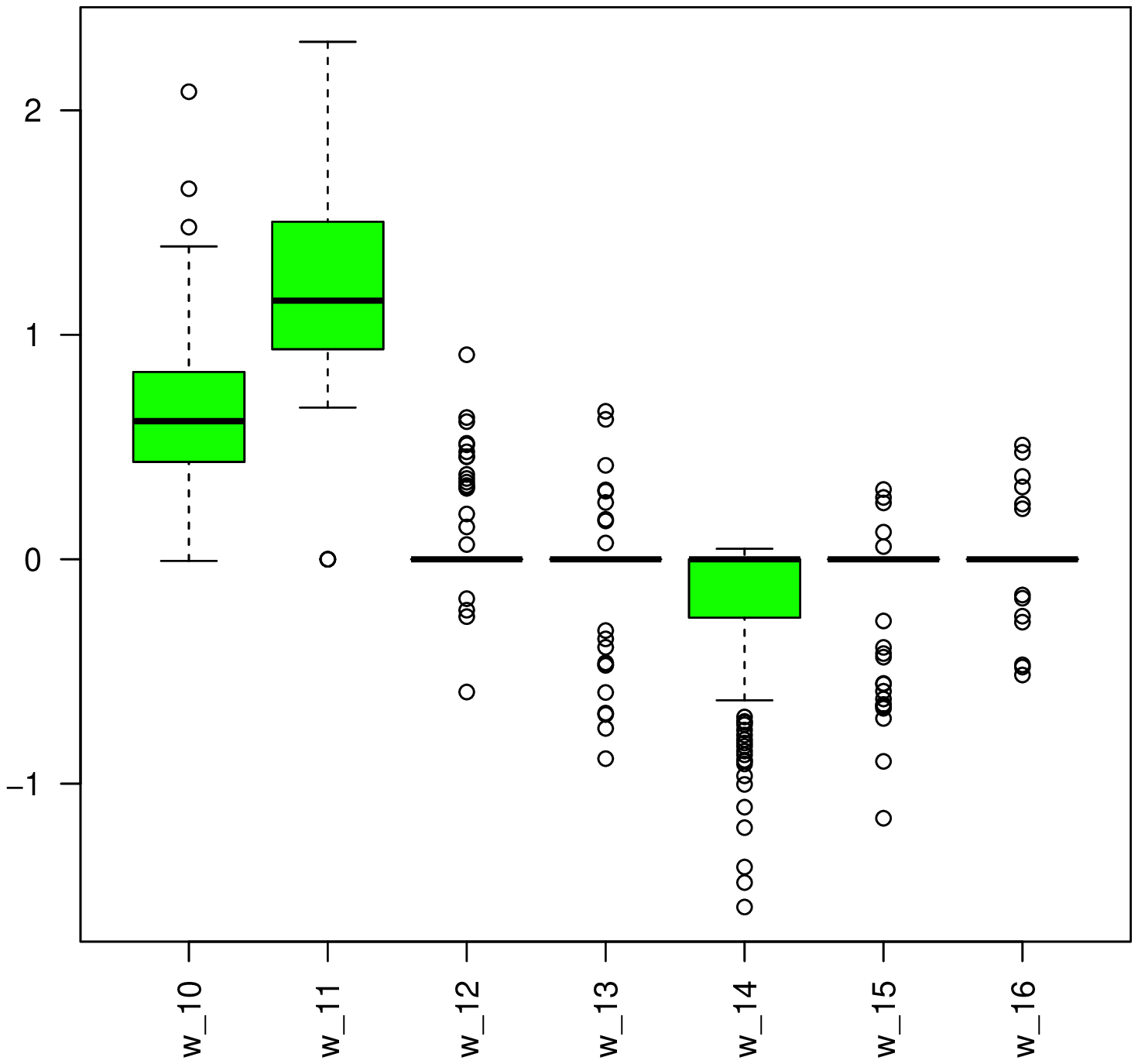}&
 \includegraphics[width = 6 cm]{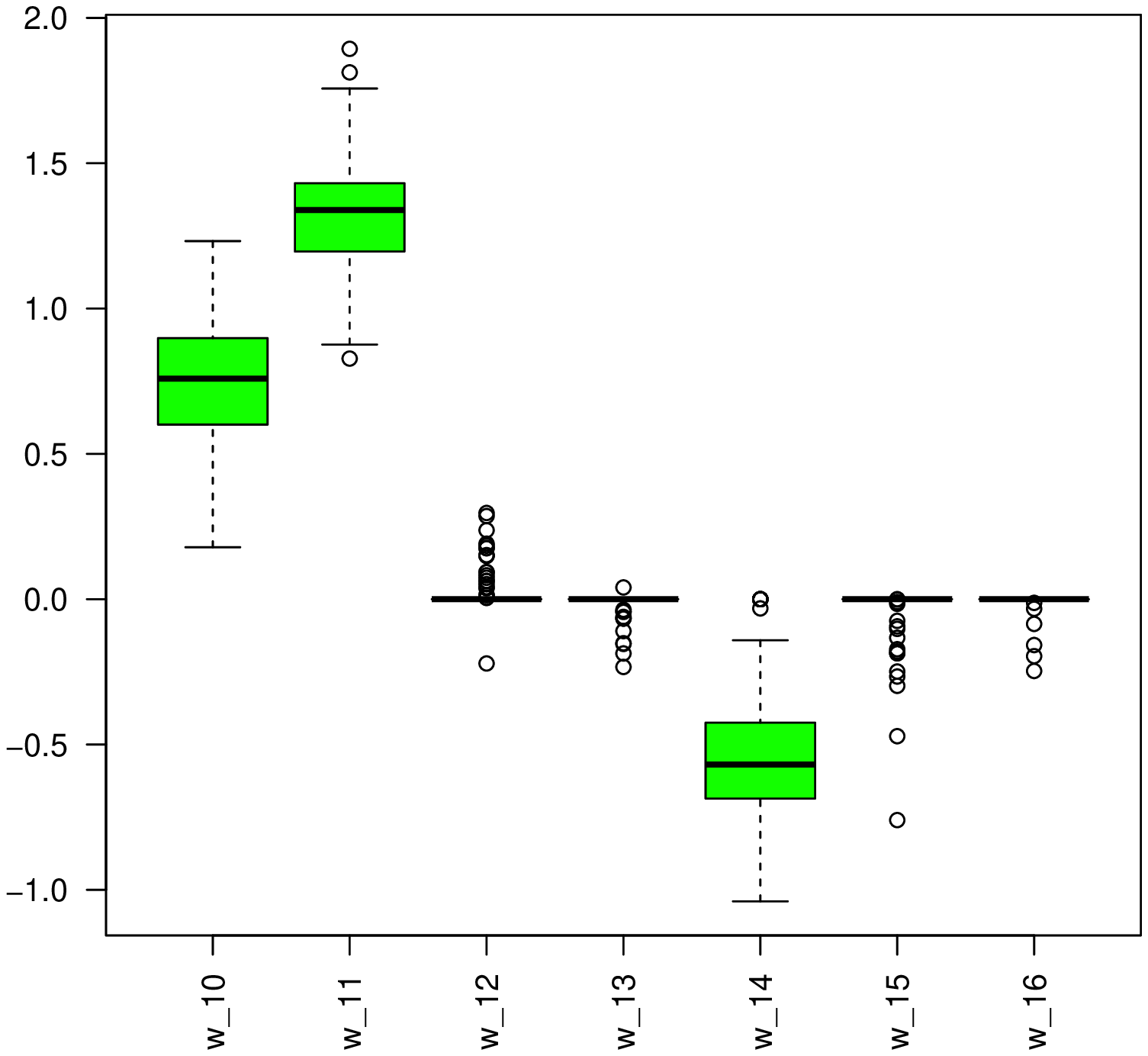} \\
MoE-Lasso + $\ell_2$ (MM) & MoE-Lasso + $\ell_2$ (CA)  
\end{tabular}
\caption{Boxplots of the gate's parameter $(w_{10},\bsw_1)^T = (1,2,0,0,-1,0,0)^T$.\label{Gate}}
\end{figure*}
For the mean and standard derivation results shown in Table \ref{MSTD}, we can see that the model using BIC for selection, the non penalized MoE, and the MoE with $\ell_2$ penalty have better results, while Lasso+$\ell_2$ and MIXLASSO can cause bias to the estimated parameters, since the penalty functions are added to the log-likelihood function. 
In contrast, from Table \ref{MSE}, in terms of average mean squared error, the Lasso+$\ell_2$ and MIXLASSO provide a better result than MoE and the MoE with $\ell_2$ penalty for estimating the zero coefficients. Between the two Lasso+$\ell_2$ algorithms, we see that the algorithm using coordinate ascent can overcome the weakness of the algorithm using MM method: once the coefficient is set to zero, it can reenter nonzero value in the progress of the EM algorithm. The BIC still provides the best result, but as we commented before, it is hard to apply BIC in reality especially for high-dimensional data, since this involves a huge collection of model candidates.
{\footnotesize\begin{table*}[!h]
\centering
\begin{tabular}{|c|c|c|c|c|c|c|c|}
\hline
Comp. & True & MoE & MoE+$\ell_2$ &MoE-BIC& Lasso+$\ell_2$ & Lasso+$\ell_2$ & MIXLASSO\\
& value & & & & (MM) & (CA) & \\
\hline
 & $0$ & $0.010_{(.096)}$ & $0.009_{(.097)}$ &$0.014_{(.083)}$ &$0.031_{(.091)}$  & $0.026_{(.089)}$ & $0.043_{(.093)}$\\
 & $0$ & $-0.002_{(.106)}$ & $-0.002_{(.107)}$ &$-0.003_{(.026)}$ &$0.009_{(.041)}$ & $0.011_{(.046)}$ & $0.011_{(.036)}$\\
 & $1.5$ & $1.501_{(.099)}$ & $1.502_{(.099)}$ &$1.495_{(.075)}$& $1.435_{(.080)}$& $1.435_{(.080)}$ & $1.404_{(.086)}$\\
Exp.1 & $0$ & $0.000_{(.099)}$ & $0.001_{(.099)}$ &$0.000_{(.037)}$&$0.012_{(.042)}$ & $0.013_{(.044)}$ & $0.013_{(.036)}$\\
 & $0$ & $-0.022_{(.102)}$ & $-0.022_{(.102)}$ &$0.002_{(.020)}$&$0.001_{(.031)}$ & $0.000_{(.032)}$ & $0.003_{(.027)}$\\
 & $0$ & $-0.001_{(.097)}$ & $-0.003_{(.097)}$ &$0.000_{(.045)}$& $0.013_{(.044)}$& $0.012_{(.043)}$ & $0.013_{(.040)}$\\
 & $1$ & $1.003_{(.090)}$ & $1.004_{(.090)}$ &$0.998_{(.077)}$& $0.930_{(.082)}$& $0.930_{(.082)}$ & $0.903_{(.088)}$\\
\hline
 & $0$ & $0.006_{(.185)}$ & $0.005_{(.184)}$ &$0.002_{(.178)}$& $-0.158_{(.183)}$& $-0.162_{(.177)}$ & $-0.063_{(.188)}$\\
 & $1$ & $1.007_{(.188)}$ & $1.006_{(.188)}$ &$1.002_{(.187)}$& $0.661_{(.209)}$& $0.675_{(.202)}$ & $0.755_{(.220)}$\\
 & $-1.5$ & $-1.492_{(.149)}$ & $-1.494_{(.149)}$ &$-1.491_{(.129)}$& $-1.216_{(.152)}$& $-1.242_{(.139)}$ & $-1.285_{(.146)}$\\
Exp.2 & $0$ & $-0.011_{(.159)}$ & $-0.012_{(.158)}$ &$-0.005_{(.047)}$& $-0.018_{(.055)}$& $-0.018_{(.055)}$ & $-0.023_{(.071)}$\\
 & $0$ & $-0.010_{(.172)}$ & $-0.008_{(.171)}$ &$-0.006_{(.079)}$& $0.013_{(.061)}$& $0.011_{(.059)}$ & $0.016_{(.075)}$\\
 & $2$ & $2.004_{(.169)}$ & $2.005_{(.169)}$ &$2.003_{(.128)}$& $1.856_{(.150)}$ & $1.876_{(.149)}$ & $1.891_{(.159)}$\\
 & $0$ & $0.008_{(.139)}$ & $0.007_{(.140)}$ &$0.008_{(.053)}$&  $0.022_{(.062)}$ & $0.020_{(.060)}$ & $0.031_{(.086)}$\\
\hline
 & $1$ & $1.095_{(.359)}$ & $1.008_{(.306)}$ &$1.055_{(.328)}$& $0.651_{(.331)}$& $0.759_{(.221)}$ & \\
 & $2$ & $2.186_{(.480)}$ & $1.935_{(.344)}$ &$2.107_{(.438)}$& $1.194_{(.403)}$& $1.332_{(.208)}$ & \\
 & $0$ & $0.007_{(.287)}$ & $0.038_{(.250)}$ &$-0.006_{(.086)}$& $0.058_{(.193)}$ & $0.024_{(.068)}$ & \\
Gate & $0$ & $-0.001_{(.383)}$ & $-0.031_{(.222)}$ &$0.004_{(.1.55)}$& $-0.025_{(.214)}$& $-0.011_{(.039)}$ & N/A\\
 & $-1$ & $-1.131_{(.413)}$ & $-0.991_{(.336)}$ &$-1.078_{(.336)}$& $-0.223_{(.408)}$ & $-0.526_{(.253)}$ & \\
 & $0$ & $-0.022_{(.331)}$ & $-0.033_{(.281)}$ &$-0.017_{(.172)}$& $-0.082_{(.243)}$& $-0.032_{(.104)}$ & \\
 & $0$ & $0.025_{(.283)}$ & $0.016_{(.246)}$ &$0.005_{(.055)}$& $-0.002_{(.132)}$& $-0.007_{(.036)}$ & \\
\hline
$\sigma$ & $1$ & $0.965_{(.045)}$ & $0.961_{(.045)}$ &$0.978_{(.046)}$ & $1.000_{(.052)}$& $0.989_{(.050)}$ & $1.000_{(.053)}$\\
\hline
\end{tabular}
\caption{Mean and standard derivation between each component of the estimated parameter vector of MoE, MoE+$\ell_2$, BIC, Lasso+$\ell_2$ (MM), Lasso+$\ell_2$ (CA) and the MIXLASSO. \label{MSTD}}
\end{table*}}
\begin{table*}[!h]
\centering
\begin{tabular}{|c|c|c|c|c|c|c|c|}
\hline
& & \multicolumn{6}{c|}{Mean squared error}\\
\cline{3-8}
Comp. & True & MoE & MoE+$\ell_2$ & MoE-BIC & Lasso+$\ell_2$ & Lasso+$\ell_2$ & MIXLASSO\\
& value & & &  & (MM) & (CA) & \\
\hline
 & $0$ & $0.0093_{(.015)}$ & $0.0094_{(.015)}$ &${0.0070_{(.011)}}$&$0.0092_{(.015)}$& $0.0087_{(.014)}$ & $0.0106_{(.016)}$\\
 & $0$ & $0.0112_{(.016)}$ & $0.0114_{(.017)}$ &${0.0007_{(.007)}}$&$0.0018_{(.005)}$& $0.0022_{(.008)}$ & $0.0014_{(.005)}$\\
 & $1.5$ & $0.0098_{(.014)}$ & $0.0098_{(.015)}$ &${0.0057_{(.007)}}$&$0.0106_{(.012)}$& $0.0107_{(.012)}$ & $0.0166_{(.019)}$\\
Exp.1 & $0$ & $0.0099_{(.016)}$ & $0.0099_{(.016)}$ &${0.0013_{(.009)}}$&$0.0019_{(.005)}$& $0.0021_{(.006)}$ & $0.0015_{(.005)}$\\
 & $0$ & $0.0108_{(.015)}$ & $0.0109_{(.016)}$ &$0.0004_{(.004)}$& $0.0010_{(.004)}$& ${0.0001_{(.004)}}$ & $0.0007_{(.003)}$\\
 & $0$ & $0.0094_{(.014)}$ & $0.0094_{(.014)}$ &$0.0020_{(.010)}$&$0.0021_{(.007)}$& $0.0020_{(.006)}$ & ${0.0017_{(.008)}}$\\
 & $1$ & $0.0081_{(.012)}$ & $0.0082_{(.012)}$ &${0.0059_{(.009)}}$&$0.0117_{(.015)}$& $0.0116_{(.015)}$ & $0.0172_{(.021)}$\\
\hline
 & $0$ & $0.0342_{(.042)}$ & $0.0338_{(.042)}$ &${0.0315_{(.049)}}$&$0.0585_{(.072)}$& $0.0575_{(.079)}$ & $0.0392_{(.059)}$\\
 & $1$ & $0.0355_{(.044)}$ & $0.0354_{(.044)}$ &${0.0350_{(.044)}}$&$0.1583_{(.157)}$& $0.1465_{(.148)}$ & $0.1084_{(.130)}$\\
 & $-1.5$ & $0.0222_{(.028)}$ & $0.0221_{(.028)}$ &${0.0166_{(.240)}}$&$0.1034_{(.098)}$& $0.0860_{(.087)}$ & $0.0672_{(.070)}$\\
Exp.2 & $0$ & $0.0253_{(.032)}$  & $0.0252_{(.031)}$ &${0.0022_{(.022)}}$&$0.0033_{(.013)}$& $0.0034_{(.017)}$ & $0.0056_{(.022)}$\\
 & $0$ & $0.0296_{(.049)}$ & $0.0294_{(.049)}$ &$0.0063_{(.032)}$&$0.0039_{(.019)}$& ${0.0037_{(.020)}}$ & $0.0059_{(.023)}$\\
 & $2$ & $0.0286_{(.040)}$ & $0.0287_{(.040)}$ &${0.0163_{(.023)}}$&$0.0432_{(.056)}$& $0.0375_{(.050)}$ & $0.0371_{(.051)}$\\
 & $0$ & $0.0195_{(.029)}$ & $0.0195_{(.029)}$ &${0.0028_{(.020)}}$&$0.0043_{(.017)}$ & $0.0040_{(.015)}$ & $0.0083_{(.028)}$\\
\hline
 & $1$ & $0.1379_{(.213)}$ & ${0.0936_{(.126)}}$ &$0.1104_{(.178)}$&$0.2315_{(.240)}$& $0.1067_{(.125)}$ & \\
 & $2$ & $0.2650_{(.471)}$ & ${0.1225_{(.157)}}$ &$0.2035_{(.371)}$&$0.8123_{(.792)}$& $0.4890_{(.277)}$ & \\
 & $0$ & $0.0825_{(.116)}$ & $0.0641_{(.086)}$ &$0.0075_{(.040)}$&$0.0404_{(.032)}$& ${0.0052_{(.015)}}$ & \\
Gate & $0$ & $0.1466_{(.302)}$ & $0.1052_{(.196)}$ & $0.0239_{(.147)}$&$0.0501_{(.050)}$& ${0.0017_{(.007)}}$ & N/A\\
 & $-1$ & $0.1875_{(.263)}$ & ${0.1129_{(.148)}}$ &$0.1189_{(.191)}$ &$0.7703_{(.760)}$& $0.2885_{(.295)}$ & \\
 & $0$ & $0.1101_{(.217)}$ & $0.0803_{(.164)}$ &$0.0299_{(.195)}$&$0.0656_{(.066)}$& ${0.0120_{(.062)}}$ & \\
 & $0$ & $0.0806_{(.121)}$& $0.0610_{(.095)}$  &$0.0030_{(.030)}$&$0.0175_{(.018)}$& ${0.0013_{(.008)}}$ & \\
\hline
$\sigma$ & $1$ & $0.0033_{(.004)}$ & $0.0035_{(.004)}$ &${0.0026_{(.003)}}$&$0.0027_{(.003)}$& $0.0027_{(.003)}$ & $0.0028_{(.003)}$\\
\hline
\end{tabular}
\caption{Mean squared error between each component of the estimated parameter vector of MoE, MoE+$\ell_2$, BIC, Lasso+$\ell_2$ (MM), Lasso+$\ell_2$ (CA) and the MIXLASSO. \label{MSE}}
\end{table*}
\subsubsection{Clustering}
We calculate the accuracy of clustering of all these mentioned models for each data set. 
The results in terms of ARI and correct classification rate values are provided in Table \ref{Cluster}. We  can see that the Lasso+$\ell_2$ (CA) model provides a good result for clustering data. The BIC model gives the best result but always with a very significant computational load. The difference between Lasso+$\ell_2$ (CA) and BIC is smaller than $1\%$, while the MIXLASSO provides a poor result in terms of clustering. Here, we also see that the Lasso+$\ell_2$ (MM) estimates the parameters in the experts quite well. However, the MM algorithm for updating the gate's parameter causes bad effect, since this approach forces the non-zero coefficient $w_{14}$ toward zero. Hence, this may decrease the clustering performance.
\begin{table}[!h]
\centering
{\small\begin{tabular}{|c|c|c|}
\hline
Model & C.rate & ARI\\
\hline
MoE & $89.57\%_{(1.65\%)}$ & $0.6226_{(.053)}$\\
MoE+$\ell_2$ & $89.62\%_{(1.63\%)}$ & $0.6241_{(.052)}$\\
MoE-BIC & ${ 90.05\%_{(1.65\%)}}$ & ${0.6380_{(.053)}}$\\
Lasso+$\ell_2$ (MM) & $87.76\%_{(2.19\%)}$ & $0.5667_{(.067)}$\\
Lasso+$\ell_2$ (CA) & $89.46\%_{(1.76\%)}$ &  $0.6190_{(.056)}$\\
MIXLASSO & $82.89\%_{(1.92\%)}$ & $0.4218_{(.050)}$\\
\hline
\end{tabular}}
\caption{Average of the accuracy of clustering (correct classification rate and Adjusted Rand Index).\label{Cluster}}
\end{table}\\
Overall, we can clearly see the Lasso+$\ell_2$ (CA) algorithm performs quite well to retrieve the actual sparse support; the sensitivity and specificity results are quite reasonable for the proposed Lasso+$\ell_2$ regularization. While the penalty function will cause bias to the parameters, as shown in the results of the MSE, the algorithm can  perform parameter density estimation  with an acceptable loss of information due to the bias induced by the regularization. 
In terms of clustering, the Lasso+$\ell_2$ (CA) works as well as two other MoE models and BIC, better than the Lasso+$\ell_2$ (MM), MIXLASSO models.
\subsection{Applications to real data sets}

We analyze two real data sets as a further test of the methodology. Here, we investigate the housing data described on the website UC Irvine Machine Learning Repository and baseball salaries from the Journal of Statistics Education (www.amstat.org/publications/jse). This was done to provide a comparison with the work of \cite{Kha10}, \cite{Kha07}. While in  \cite{Kha07} the authors used Lasso-penalized mixture of linear regression (MLR) models, we still apply penalized mixture of experts (to better represent  the data than when using MRL models). 
We compare the results of each model based upon two different criteria: the average mean squared error (MSE) between observation values of the response variable and the predicted values of this variable; we also consider the correlation of these values. After the parameters are estimated, the following expected value under the estimated model
\begin{align*}
\mathbb{E}_{\hat{\bstheta}}(Y|\bsx) &= \sum\limits_{k=1}^K\pi_k(\bsx;\hat{\bsw})\mathbb{E}_{\hat{\bstheta}}(Y|Z=k,\bsx)\\
 & = \sum\limits_{k=1}^K\pi_k(\bsx;\hat{\bsw})(\hat{\beta}_{k0}+\bsx^T\hat{\bsbeta}_k),
 \end{align*}
is used as a predicted value for $Y$. 
We note that here for the real data we do not consider the MoE model with BIC selection since it is computationally expensive.
\subsubsection{Housing data}
This data set concerns houses' value in the suburbs of Boston. 
It contains $506$ observations and $13$ features that may affect the house value. 
These features are: Per capita crime rate by town $(x_1)$; proportion of residential land zoned for lots over $25,000$ sq.ft. $(x_2)$; proportion of non-retail business acres per town $(x_3)$;  Charles River dummy variable ($= 1$ if tract bounds river; $0$ otherwise) $(x_4)$; nitric oxides concentration (parts per 10 million) $(x_5)$; average number of rooms per dwelling $(x_6)$; proportion of owner-occupied units built prior to $1940$ $(x_7)$; weighted distances to five Boston employment centres $(x_8)$; index of accessibility to radial highways $(x_9)$; full-value property-tax rate per $\$10,000$ $(x_{10})$;  pupil-teacher ratio by town $(x_{11})$; $1000(Bk - 0.63)^2$ where $Bk$ is the proportion of blacks by town $(x_{12})$;  $\%$ lower status of the population $(x_{13})$. 
The columns of $X$ were standardized to have mean $0$ and variance $1$. 
The response homes in variable of interest is the median value of owner occupied homes in $\$1000's$, MEDV. Based on the histogram of $Y$ = MEDV/sd(MEDV), where sd(MEDV) is the standard deviation of MEDV, Khalili decided to separate $Y$ into two groups of houses with $``$low$"$ and $``$high$"$ values. Hence, a MoE model is used to  fit the response
$$Y\sim\pi_1(\bsx;\bsw)\mathcal{N}(y; \beta_{10}+ \bsx^T\bsbeta_1,\sigma^2) + (1-\pi_1(\bsx;\bsw))\mathcal{N}(y; \beta_{20}+ \bsx^T\bsbeta_2,\sigma^2),$$ 
where $\displaystyle\pi_1(\bsx;\bsw) = \frac{e^{w_{10}+\bsx^T \bsw_1}}{1+e^{w_{10}+\bsx^T \bsw_1}}$.
The parameter estimates of the MoE models obtained by Lasso+$\ell_2$ and MLE are given in Table \ref{House}. We compare our results with those of Khalili and the non-penalized MoE. 
\begin{table}
\begin{center}
\begin{tabular}{|c|c|c|c|c|c|c|c|c|c|}
\hline
Features & \multicolumn{3}{|c}{MLE, $\hat{\sigma} = 0.320$} & \multicolumn{3}{|c}{Lasso+$\ell_2$ (Khalili), $\hat{\sigma} = 0.352$} & \multicolumn{3}{|c|}{Lasso+$\ell_2$, $\hat{\sigma} = 0.346$}\\
\cline{2-10}
\ & Exp.1 & Exp.2 & Gate& Exp.1 & Exp.2 & Gate& Exp.1 & Exp.2 & Gate\\
\hline
$x_0$ & 2.23 & 3.39 & 19.17 & 2.16 & 2.84 & 1.04 & 2.20 & 2.82 & 0.79\\
$x_1$ & -0.12 & 3.80 & -4.85 & -0.09 & - & - & -0.09 & - & -\\
$x_2$ & 0.07 & 0.04 & -5.09 & - & 0.07 & - & - & 0.07 & -\\
$x_3$ & 0.05& -0.03 & 7.74 & - & - & 0.67 & - & - & 0.41\\
$x_4$ & 0.03 & -0.01 & -1.46 & - & 0.05 & - & 0.05 & 0.06 & - \\
$x_5$ & -0.18 & -0.16 & 9.39 & - & - & - & -0.08 & - & - \\
$x_6$ & -0.01 & 0.63 & 1.36 & - & 0.60 & -0.27 & - & 0.56 & - \\
$x_7$ & -0.06 & -0.07 & -8.34 & - & - & - & -0.05 & - & - \\
$x_8$ &-0.20 & -0.21 & 8.81 & - & -0.20 & - &-0.03 & -0.19 & - \\
$x_9$ & 0.02 & 0.31 & 0.96 & - &0.55 & - & - & 0.60 & - \\
$x_{10}$ &-0.19 & -0.33 & -0.45 & - & - & - & -0.01 & - & - \\
$x_{11}$ &-0.14 & -0.18 & 7.06 & - & - & 0.54 & -0.10 & -0.08 & 0.28\\
$x_{12}$ &0.06 & 0.01 & -6.17 & 0.05 & - & - & 0.05 & - & - \\
$x_{13}$ &-0.32 & -0.73 & 36.27& -0.29 & -0.49 & 1.56 & -0.29 &-0.57 & 1.05 \\ 
\hline
\end{tabular}
\caption{Fitted models for housing data.\label{House}}
\end{center}
\end{table}
In Table \ref{ResultHousing}, we provide the result in terms of average MSE and the correlation between the true observation value $Y$ and its prediction $\hat{Y}$. Our result provides a least sparse model than Khalili's. Some parameters in both methods have the same value. However, the MSE and the correlation from our method are better than those of Khalili. Hence, in application one would consider the sparsity and the prediction of each estimated parameters. Both Lasso+$\ell_2$ algorithms give comparative results with the MLE.
\begin{table}[!h]
    \centering
    \begin{tabular}{|c|c|c|c|}
    \hline
         & MoE & Lasso+$\ell_2$ (Khalili) & Lasso+$\ell_2$ \\
         \hline
        $R^2$ & 0.8457 & 0.8094 & 0.8221\\
        \hline
        MSE & $0.1544_{(.577)}$ & $0.2044_{(.709)}$ & $0.1989_{(.619)}$\\
        \hline
    \end{tabular}
    \caption{Results for Housing data set.}
    \label{ResultHousing}
\end{table}
\subsubsection{Baseball salaries data}
We now consider baseball salaries data set from the Journal of Statistics Education (see also \cite{Kha07}) as a further test of the methodology. 
This data set contains $337$ observations and $33$ features. 
We compare our results with the non-penalized MoE models and the MIXLASSO models (see \cite{Kha07}). 
 \cite{Kha07} used this data set in the analysis, which included an addition of $16$ interaction features, making in total $32$ predictors. The columns of $\bsX$ were standardized to have mean $0$ and variance $1$. Histogram of the log of salary shows multimodality making it a good candidate for the response variable under the MoE model with two components:
$$Y = \log(salary)\sim  \pi_1(\bsx;\bsw)\mathcal{N}(y; \beta_{10}+ \bsx^T\bsbeta_1,\sigma^2) 
 + (1-\pi_1(\bsx;\bsw))\mathcal{N}(y; \beta_{20}+ \bsx^T\bsbeta_2,\sigma^2).$$
By taking all the tuning parameters to zero, we obtain the maximum likelihood estimator of the model. We also compare our result with MIXLASSO from \cite{Kha07}. 
Table \ref{TableBB} presents the estimated parameters for baseball salary data and Table \ref{ResultBB} shows the results in terms of MSE, and $R^2$ between the true value of $Y$ and its predicted value. 
These results suggest that the proposed algorithm with the Lasso+$\ell_2$ penalty also shrinks some parameters to zero and have acceptable results compared to MoE. It also shows that this model provides better results than that of the MIXLASSO model. 
%
\begin{table*}[!h]
\begin{center}
\begin{tabular}{|c|c|c|c|c|c|c|c|c|}
\hline
Features & \multicolumn{3}{|c}{MLE, $\widehat{\sigma} = 0.277$} & \multicolumn{3}{|c}{Lasso+$\ell_2$, $\widehat{\sigma} = 0.345$} & \multicolumn{2}{|c|}{MIXLASSO, $\widehat{\sigma} = 0.25$}\\
\cline{2-9}
\ & Exp.1 & Exp.2 & Gate & Exp.1 & Exp.2 & Gate & Exp.1 & Exp.2\\
\hline
$x_0$ & 6.0472 & 6.7101 & -0.3958 & 5.9580 & 6.9297 & 0.0046 & 6.41 & 7.00\\
$x_1$ & -0.0073 & -0.0197 & 0.1238 & -0.0122 & - & - & - & -0.32\\
$x_2$ & -0.0283 & 0.1377 & 0.1315 & -0.0064 & - & - & - & 0.29\\
$x_3$ & 0.0566 & -0.4746 & 1.5379 & - & - & - & - & -0.70\\
$x_4$ & 0.3859 & 0.5761 & -1.9359 & 0.4521 & 0.0749 & - & 0.20 & 0.96\\
$x_5$ & -0.2190 & -0.0170 & -0.9687 & - & - & - & - & -\\
$x_6$ & -0.0586 & 0.0178 & 0.4477 & -0.0051 & - & - & - & -\\
$x_7$ & -0.0430 & 0.0242 & -0.3682 & - & - & - & -0.19 & -\\
$x_8$ & 0.3991 & 0.0085 & 1.7570 & - &0.0088 & - & 0.26 & -\\
$x_9$ & -0.0238 & -0.0345 & -1.3150 & 0.0135 & 0.0192 & - & - & -\\
$x_{10}$ & -0.1944 & 0.0412 & 0.6550 & -0.1146 & - & - & - & -\\
$x_{11}$ & 0.0726 & 0.1152 & 0.0279 & -0.0108 & 0.0762 & - & - & -\\
$x_{12}$ & 0.0250 & -0.0823 & 0.1383 & - & - & - & - & -\\
$x_{13}$ & -2.7529 & 1.1153 & -7.0559 & - & 0.3855 &-0.3946& 0.79 & 0.70\\
$x_{14}$ & 2.3905 & -1.4185 & 5.6419 & 0.0927 & -0.0550 & - & 0.72 & -\\
$x_{15}$ & -0.0386 & 1.1150 & -2.8818 & 0.3268 & 0.3179 & - & 0.15 & 0.50\\ 
$x_{16}$ & 0.2380 & 0.0917 & -7.9505 & - & - & - & - & -0.36\\
$x_1*x_{13}$ & 3.3338 & -0.8335 & 8.7834 & 0.3218 & - & - & -0.21 & -\\ 
$x_1*x_{14}$ & -2.4869 & 2.5106 & -7.1692 & - & - & - & 0.63 & -\\ 
$x_1*x_{15}$ & 0.4946 & -0.9399 & 2.6319 & - & - & - & 0.34 & -\\
$x_1*x_{16}$ & -0.4272 & -0.4151 & 7.9715 & -0.0319 & - & - & - & -\\
$x_3*x_{13}$ & 0.7445 & 0.3201 & 0.5622 & - & 0.0284 & -0.5828 & - & -\\
$x_3*x_{14}$ & -0.0900 & -1.4934 & 0.1417 & -0.0883 & - & - & 0.14 & -0.38\\
$x_3*x_{15}$ & -0.2876 & 0.4381 & -0.9124 & - & - & - & - & -\\
$x_3*x_{16}$ & -0.2451 & -0.2242 & -5.6630 & - & - & - & -0.18 & 0.74\\
$x_7*x_{13}$ & 0.7738 & 0.1335 & 4.3174 & - & 0.004 & - & - & -\\
$x_7*x_{14}$ & -0.1566 & 1.2809 & -3.5625 & -0.1362 & 0.0245 & - & - & -\\
$x_7*x_{15}$ & -0.0104 & 0.2296 & -0.4348 & - & - & - & - & 0.34\\
$x_7*x_{16}$ & 0.5733 & -0.2905 & 3.2613 & - & - & - & - & -\\
$x_8*x_{13}$ & -1.6898 & -0.0091 & -8.7320 & - & 0.2727 & -0.3628 & 0.29 & -0.46\\
$x_8*x_{14}$ & 0.7843 & -1.3341 & 6.2614 & - & 0.0133 & - & -0.14 & -\\
$x_8*x_{15}$ & 0.3711 & -0.4310 & 0.8033 & 0.3154 & - & - & - & -\\
$x_8*x_{16}$ & -0.2158 & 0.7790 & 2.6731 & 0.0157 & - & - & - & -\\
\hline
\end{tabular}
\end{center}
\caption{Fitted models for baseball salary data. \label{TableBB}}
\end{table*}
\begin{table}[!h]
    \centering
    \begin{tabular}{|c|c|c|c|}
    \hline
         & MoE & Lasso+$\ell_2$ & MIXLASSO\\
         \hline
        $R^2$ & 0.8099 & 0.8020 & 0.4252\\
        \hline
        MSE & $0.2625_{(.758)}$ & $0.2821_{(.633)}$ & $1.1858_{(2.792)}$\\
        \hline
    \end{tabular}
    \caption{Results for Baseball salaries data set.}
    \label{ResultBB}
\end{table}

\section{Discussion for the high-dimensional setting}\label{sec:Discuss}
Indeed, the developed MM and coordinate ascent algorithms for the estimation of the parameters of our model could be slow in a high-dimensional setting since we do not have the closed-form updates of the parameters of the gating network $\bsw$ at each step of the EM algorithm; while a univariate Newton-Raphson is derived to avoid matrix inversion operations, it is still slow in high-dimension. 
However,  as we very recently developed it, this difficulty can be overcome by a proximal Newton algorithm. The idea is that, for updating the parameters of the gating network $\bsw$,  rather than maximizing $Q(\bsw;\bf{\theta}^{[q]})$ which is non-smooth and non-quadratic, we maximize an  approximate of the smooth part of $Q(\bsw;\bf{\theta}^{[q]})$ by its local quadratic form by using Taylor expansion around the current parameter estimate, $\tilde{Q}(\bsw;\bf{\theta}^{[q]})$. For more details on the proximal Newton methods, we refer to \cite{S.Lee06}, \cite{Friedman10} and \cite{Lee14}. 
%
The resulting proximal function $\tilde{Q}(\bsw;\bf{\theta}^{[q]})$ is
 then maximized, by using a
 coordinate ascent algorithm, but which has a closed-form update at each step, and thus also still avoid computing matrix inversions. 
Hence, this new algorithm improves the running time of the EMM algorithm with MM and coordinate ascent algorithm, and performs quite well in a high-dimensional setting. 
The R code we publicly provide also contains this version.

To evaluate the algorithm in a situation in which we have a high number of features, we consider the Residential Building Data Set (UCI Machine Learning Repository). 
This data set contains $372$ and $108$ features with the two response variables (V-9 and V-10), which represent the sale prices and construction costs. 
We choose the V-9 variable (sale prices) as the response variable to be predicted.
All the features are standardized to have zero-mean and unit-variance. 
We provide the results of our algorithm with $K=3$ expert components and $\lambda = 15$, $\gamma = 5$. 
The estimated parameters are given in Table \ref{RBData1} and \ref{RBData2}.
The correlation and the mean squared error between the true value V-9 with its prediction can be found in Table \ref{RBResult}.
These results show that the proximal Newton method performs well in this setting, in which it provides a sparse model and competitive criteria in prediction and clustering.  
\begin{table}[!h]
\footnotesize
\begin{center}
\begin{tabular}{|c|c|c|c||c|c|}
\hline
Features & \multicolumn{3}{c||}{Expert,\ $\sigma = 0.0255$}&\multicolumn{2}{c|}{Gating network}\\
\cline{2-6}
& Exp.1 & Exp.2 & Exp.3 & Gate.1 & Gate.2\\
\hline
$x_{0}$	&-0.00631	&-0.01394	&-0.07825	&0.43542	&2.40874\\
$x_{1}$	&-	&-	&0.00599	&-	&-\\
$x_{2}$	&0.02946	&-0.00442	&-	&-	&-\\
$x_{3}$	&-	&-	&0.00849	&-	&-\\
$x_{4}$	&-0.00776	&0.00406	&0.01485	&-	&-\\
$x_{5}$	&-0.00619	&-0.00759	&-0.04185	&-0.23943	&-\\
$x_{6}$	&0.00125	&0.02581	&-	&-	&-\\
$x_{7}$	&-	&-0.01823	&0.00233	&-	&-\\
$x_{8}$	&0.02271	&-0.01962	&0.01964	&-0.04267	&-\\
$x_{9}$	&0.06822	&0.00274	&0.02101	&-	&-\\
$x_{10}$	&-0.03166	&-0.00008	&-	&-	&-\\
$x_{11}$	&0.12789	&0.05117	&0.03515	&-	&-0.91114\\
$x_{12}$	&1.10946	&1.00213	&0.78915	&0.22049	&-0.71761\\
$x_{13}$	&0.00878	&-0.00647	&-	&0.41648	&-\\
$x_{14}$	&-	&-	&-	&-	&-\\
$x_{15}$	&-	&-	&-	&-	&-\\
$x_{16}$	&-0.01495	&-0.00103	&0.03774	&-	&-\\
$x_{17}$	&-	&-	&-	&-	&-\\
$x_{18}$	&-	&-0.03344	&-	&-	&-\\
$x_{19}$	&-	&0.06296	&-	&-	&-\\
$x_{20}$	&0.04560	&0.02466	&-	&-	&-\\
$x_{21}$	&0.02368	&0.03210	&-	&-	&-\\
$x_{22}$	&-	&-0.00546	&-0.00398	&-	&-\\
$x_{23}$	&-	&-0.03934	&-	&-	&-\\
$x_{24}$	&-	&-0.04612	&-	&-	&-\\
$x_{25}$	&0.01205	&-0.00352	&-	&-	&-\\
$x_{26}$	&-	&-	&-	&-	&-\\
$x_{27}$	&-	&0.00409	&-	&-	&-\\
$x_{28}$	&-	&-	&-	&-	&-\\
$x_{29}$	&-	&-	&0.00047	&-	&-\\
$x_{30}$	&-	&-	&-	&-	&-\\
$x_{31}$	&-	&0.03494	&0.04131	&-	&-\\
$x_{32}$	&-	&-0.00003	&0.02288	&-	&-\\
$x_{33}$	&-	&-	&-	&-	&-\\
$x_{34}$	&-	&-	&-	&-	&-\\
$x_{35}$	&-	&0.01468	&-0.01095	&-	&-\\
$x_{36}$	&-	&-	&-	&-	&-\\
$x_{37}$	&-	&0.00899	&-	&-	&-\\
$x_{38}$	&-	&0.00061	&-	&-	&-\\
$x_{39}$	&-0.01694	&-0.00559	&-	&-	&-\\
$x_{40}$	&0.10214	&0.02533	&-	&0.07086	&-\\
$x_{41}$	&0.03770	&-	&-	&-	&-\\
$x_{42}$	&-	&-0.04162	&-	&-	&-\\
$x_{43}$	&-	&-	&-	&-	&-\\
$x_{44}$	&-	&0.00561	&0.01148	&-	&-\\
$x_{45}$	&-	&0.00770	&-	&-	&-\\
$x_{46}$	&-	&-	&-	&-	&-\\
$x_{47}$	&-	&-	&-	&-	&-\\
$x_{48}$	&-0.07316	&0.03138	&-	&-	&-\\
$x_{49}$	&-	&0.00493	&-0.00183	&-	&-\\
$x_{50}$	&-	&0.01320	&-	&-	&-\\
$x_{51}$	&-0.00076	&-0.00041	&-	&-	&0.03819\\
$x_{52}$	&-	&-	&-	&-	&-\\
$x_{53}$	&-	&-	&-	&-	&-\\
\hline
\end{tabular}
\caption{Fitted model parameters for residential building data (part 1).}\label{RBData1}
\end{center}
\end{table}
\begin{table}[!h]
\footnotesize
\begin{center}
\begin{tabular}{|c|c|c|c||c|c|}
\hline
Features & \multicolumn{3}{c||}{Expert,\ $\sigma = 0.0255$}&\multicolumn{2}{c|}{Gating network}\\
\cline{2-6}
& Exp.1 & Exp.2 & Exp.3 & Gate.1 & Gate.2\\
\hline
$x_{54}$	&-0.00854	&0.00077	&-	&-	&-\\
$x_{55}$	&-	&0.00039	&-	&-	&-\\
$x_{56}$	&-	&-	&-0.11177	&-	&-\\
$x_{57}$	&-	&0.00334	&-	&-	&-\\
$x_{58}$	&0.04779	&0.00405	&0.00733	&0.35226	&-\\
$x_{59}$	&0.06726	&0.03743	&0.02988	&0.08489	&-0.20694\\
$x_{60}$	&0.02520	&0.00128	&0.01473	&-	&-\\
$x_{61}$	&-	&0.00843	&-	&-	&-\\
$x_{62}$	&-	&0.00034	&-	&-	&-\\
$x_{63}$	&-	&-0.00920	&0.01184	&-	&-\\
$x_{64}$	&-	&0.00002	&-	&-	&-\\
$x_{65}$	&-	&-	&-	&-	&-\\
$x_{66}$	&-	&-	&-	&-	&-\\
$x_{67}$	&-0.03840	&-	&0.02505	&-	&-\\
$x_{68}$	&-	&0.00234	&0.00238	&-	&-\\
$x_{69}$	&-	&-	&-	&-	&-\\
$x_{70}$	&0.06026	&0.01750	&0.05879	&-	&-\\
$x_{71}$	&-	&-	&-	&-	&-\\
$x_{72}$	&-	&-0.03636	&-	&-	&-\\
$x_{73}$	&-	&-	&-0.02932	&-	&-\\
$x_{74}$	&-	&-	&-	&-	&-\\
$x_{75}$	&-0.02725	&-0.02474	&-	&-	&-\\
$x_{76}$	&-0.01399	&-0.16005	&-0.08654	&-	&-\\
$x_{77}$	&-	&0.00526	&-	&-	&-\\
$x_{78}$	&-0.05816	&0.02821	&-	&0.01303	&-0.35566\\
$x_{79}$	&-	&-0.00358	&-	&1.12522	&-\\
$x_{80}$	&-0.05416	&-	&-	&-	&-\\
$x_{81}$	&-	&-	&-	&-	&-\\
$x_{82}$	&-	&-	&0.04329	&-	&-\\
$x_{83}$	&-	&-	&-	&-	&-\\
$x_{84}$	&-	&-	&-	&-	&-\\
$x_{85}$	&-	&-	&-	&-	&-\\
$x_{86}$	&-	&0.00783	&-	&-	&-\\
$x_{87}$	&-	&-	&0.01463	&-	&-\\
$x_{88}$	&0.02337	&0.03903	&-	&-	&-\\
$x_{89}$	&-0.04720	&0.00909	&-	&-	&-\\
$x_{90}$	&-	&-	&-	&-	&-\\
$x_{91}$	&-	&-	&-	&-	&-\\
$x_{92}$	&-0.00070	&-0.00626	&-0.00458	&-	&-\\
$x_{93}$	&-	&-	&-	&-	&-\\
$x_{94}$	&-0.00067	&0.00309	&-	&-	&-\\
$x_{95}$	&-	&-0.00925	&-	&-	&-\\
$x_{96}$	&-0.00705	&-0.00656	&-	&-	&0.03610\\
$x_{97}$	&-	&-0.00406	&-	&-	&-\\
$x_{98}$	&-	&0.00714	&0.01911	&0.06610	&-\\
$x_{99}$	&-	&0.00364	&-	&-	&-\\
$x_{100}$	&-	&0.00327	&-	&-	&-\\
$x_{101}$	&-	&0.02858	&0.03974	&-	&-\\
$x_{102}$	&0.01623	&-0.01236	&-	&-	&-\\
$x_{103}$	&-	&-	&-	&-	&-\\
$x_{104}$	&-	&-	&-	&-	&-\\
$x_{105}$	&-	&0.00215	&-	&-	&-\\
$x_{106}$	&-0.00006	&-0.00129	&-	&-	&-\\
$x_{107}$	&-	&0.00851	&-	&-	&-\\
\hline
\end{tabular}
\caption{Fitted model parameters for residential building data (part 2).}\label{RBData2}
\end{center}
\end{table}
\begin{table}[!h]
    \centering
    \begin{tabular}{|c|c|c||c|c|c|c|c|}
    \hline
     & \multicolumn{2}{c||}{Predictive criteria}&\multicolumn{5}{c|}{Number of zero coefficients}\\
    \cline{2-8}
 Method & $R^2$ & MSE & Exp.1 & Exp.2 & Exp.3 & Gate.1 & Gate.2\\
\hline
Proximal Newton & 0.991 & $0.0093_{(.059)}$ & 71 & 38 & 75 & 97 & 101\\
\hline
\end{tabular}
\caption{Results for residential building data set.}
\label{RBResult}
\end{table}
We also provide the correlation and the mean squared error between those values after clustering the data in Table \ref{RBClustering}.
\begin{table}[!h]
    \centering
    \begin{tabular}{|c|c|c||c|c|c|}
\hline
& \multicolumn{2}{c||}{Predictive criteria}&\multicolumn{3}{c|}{Number of observations}\\
\cline{2-6}
Method & $R^2$ & MSE & Class 1 & Class 2 & Class 3\\
\hline
 Proximal Newton & 0.9994 & $0.00064_{(.0018)}$ & 59 & 292 & 21\\
\hline
\end{tabular}
\caption{Results for clustering the residential building data set.}
\label{RBClustering}
\end{table}
For the CPU times, we compare two methods: the coordinate ascent algorithm (CA) and the proximal Newton method (PN). We test these algorithms on different data sets. 
The first one is the one of $100$ data sets used for the simulation study. With this data set, we run these algorithms $10$ times and the number of clusters $K=2$ and $K=3$. 
\\
The second data set is the baseball salaries. 
Finally, we also consider the residential building data set as a further comparison with the proximal Newton method. 
%
The computer used for this work has CPU Intel i5-6500T 2.5GHz with 16GB RAM. 
The obtained results are given in Table \ref{CPU}.
We can see that the algorithm for the residential data which has a quite high number of features, requires only few minutes and is thus has a very reasonable speed, and for moderate dimensional problems, is very fast. 
\begin{table}[!h]
\footnotesize
    \centering
    \begin{tabular}{|c|c|c|c|c|c|}
\hline
Data & No. features & No. observations & No. experts & CA & PN\\
\hline
Simulation &$7$ & $300$ & $2$ & $45.34_{(14.28)}$ (s) & $5.03_{(1.09)}$ (s)\\
Simulation &$7$ & $300$ & $3$ & $7.94_{(13.22)}$ (m) & $20.52_{(9.23)}$ (s)\\
Baseball salaries & $33$ & $337$ & $2$ & $17.9_{(15.87)}$ (m) & $46.76_{(21.02)}$ (s)\\
Residential Data & $108$ & $372$ & $3$ & N/A & $3.63_{(0.58)}$ (m)\\
\hline
\end{tabular}
\caption{Results for CPU times.}
\label{CPU}
\end{table}
\\

\noindent{\bf An experiment for $d > n$:} To consider the high-dimensional setting, we take the first $n=90$ observations of the  residential building data  with all the $d=108$ features. 
We use a mixture of three experts and provide the results by applying the 
 proximal Newton method of the algorithm.
%
 The parameter estimation results are provided in Table \ref{ResizeRBData1} and Table \ref{ResizeRBData2}. 
The results in terms of correlation and the mean squared error between the true value V-9 and its prediction, are given in Table \ref{ResizeRBResult} and  Table \ref{ResizeRBClustering}.
\begin{table}[!h]
    \centering
    \begin{tabular}{|c|c|c||c|c|c|c|c|}
    \hline
     & \multicolumn{2}{c||}{Predictive criteria}&\multicolumn{5}{c|}{Number of zero coefficients}\\
    \cline{2-8}
 Method & $R^2$ & MSE & Exp.1 & Exp.2 & Exp.3 & Gate.1 & Gate.2\\
\hline
Proximal Newton & 0.9895 & $0.0204_{(.056)}$ & 31 & 60 & 55 & 106 & 104\\
\hline
\end{tabular}
\caption{Results for the subset of the residential building data set.}
\label{ResizeRBResult}
\end{table}

\begin{table}[!h]
    \centering
    \begin{tabular}{|c|c|c||c|c|c|}
\hline
& \multicolumn{2}{c||}{Predictive criteria}&\multicolumn{3}{c|}{Number of observations}\\
\cline{2-6}
Method & $R^2$ & MSE & Class 1 & Class 2 & Class 3\\
\hline
 Proximal Newton & 0.9999 & $0.00025_{(.0014)}$ & 63 & 11 & 16\\
\hline
\end{tabular}
\caption{Results for clustering the subset of residential building data set.}
\label{ResizeRBClustering}
\end{table}

\begin{table}[!h]
\footnotesize
\begin{center}
\begin{tabular}{|c|c|c|c||c|c|}
\hline
Features & \multicolumn{3}{c||}{Expert,\ $\sigma = 0.0159$}&\multicolumn{2}{c|}{Gating network}\\
\cline{2-6}
& Exp.1 & Exp.2 & Exp.3 & Gate.1 & Gate.2\\
\hline
$x_{0}$	&0.09048	&0.21992	&0.05460	&0.73646	&-0.54048\\
$x_{1}$	&-	&-	&-	&-	&-\\
$x_{2}$	&0.00837	&-	&0.00112	&-	&-\\
$x_{3}$	&-	&-	&-	&-	&-\\
$x_{4}$	&0.04498	&0.07325	&0.00001	&-	&-\\
$x_{5}$	&0.08075	&0.00807	&0.00010	&-	&-\\
$x_{6}$	&-0.00836	&-	&-0.02235	&0.02205	&-\\
$x_{7}$	&0.01337	&-0.00009	&-0.00922	&-	&-\\
$x_{8}$	&0.02375	&0.00443	&0.00668	&-	&-\\
$x_{9}$	&0.02194	&0.00379	&-0.03344	&-	&-\\
$x_{10}$	&-0.01305	&-0.00079	&0.00560	&-	&-\\
$x_{11}$	&0.12763	&0.01256	&0.08537	&-	&0.16264\\
$x_{12}$	&1.08977	&0.72843	&1.04263	&-	&-\\
$x_{13}$	&0.00171	&0.09792	&-	&-	&-\\
$x_{14}$	&-0.03158	&-	&-	&-	&-\\
$x_{15}$	&-	&-	&-0.00001	&-	&-\\
$x_{16}$	&-0.02218	&0.00987	&-0.00527	&-	&-\\
$x_{17}$	&-	&-	&-	&-	&-\\
$x_{18}$	&-	&-	&-0.10258	&-	&-\\
$x_{19}$	&-0.06036	&-	&-	&-	&-\\
$x_{20}$	&0.03513	&-	&-0.00602	&-	&-\\
$x_{21}$	&0.01947	&0.12495	&0.07810	&-	&-\\
$x_{22}$	&-0.00347	&0.01317	&-	&-	&-\\
$x_{23}$	&-0.03255	&-0.00125	&-	&-	&-\\
$x_{24}$	&-0.06659	&-0.00007	&-	&-	&-\\
$x_{25}$	&0.03478	&-	&0.01314	&-	&-\\
$x_{26}$	&0.01209	&0.03787	&-0.00287	&-	&-\\
$x_{27}$	&-	&-	&-	&-	&-\\
$x_{28}$	&-	&-	&-	&-	&-\\
$x_{29}$	&0.06476	&0.02369	&-0.00461	&-	&-\\
$x_{30}$	&-0.01017	&-0.00813	&0.01805	&-	&-\\
$x_{31}$	&0.03331	&-	&-	&-	&-\\
$x_{32}$	&-0.03870	&0.01708	&-	&-	&-\\
$x_{33}$	&-	&-	&-	&-	&-\\
$x_{34}$	&-	&-	&-	&-	&-\\
$x_{35}$	&0.02278	&-0.02794	&0.01933	&-	&-\\
$x_{36}$	&-	&-	&-	&-	&-\\
$x_{37}$	&-0.09359	&-	&-0.06125	&-	&-\\
$x_{38}$	&-	&-	&-0.00356	&-	&-\\
$x_{39}$	&-0.11611	&-	&-0.01973	&-	&-\\
$x_{40}$	&0.21178	&0.06134	&0.13879	&-	&-\\
$x_{41}$	&0.09095	&-	&-	&-	&-\\
$x_{42}$	&-0.03243	&-	&-	&-	&-\\
$x_{43}$	&-0.00032	&-	&-0.01455	&-	&-\\
$x_{44}$	&-0.01643	&-	&-	&-	&-\\
$x_{45}$	&-0.03152	&0.01812	&-0.02303	&-	&-\\
$x_{46}$	&-	&-	&-	&-	&-\\
$x_{47}$	&-	&-	&-	&-	&-\\
$x_{48}$	&0.13661	&0.00862	&-	&-	&-\\
$x_{49}$	&0.04914	&0.06704	&-	&-	&-\\
$x_{50}$	&0.00424	&-	&-0.02954	&-	&-\\
$x_{51}$	&0.04225	&0.05518	&-0.01411	&-	&-\\
$x_{52}$	&-	&-	&-	&-	&-\\
$x_{53}$	&-0.01697	&-	&-	&-	&-\\
\hline
\end{tabular}
\caption{Fitted model parameters for the subset of residential building data (part 1).}\label{ResizeRBData1}
\end{center}
\end{table}

\begin{table}[!h]
\footnotesize
\begin{center}
\begin{tabular}{|c|c|c|c||c|c|}
\hline
Features & \multicolumn{3}{c||}{Expert,\ $\sigma = 0.0159$}&\multicolumn{2}{c|}{Gating network}\\
\cline{2-6}
& Exp.1 & Exp.2 & Exp.3 & Gate.1 & Gate.2\\
\hline
$x_{54}$	&0.02922	&0.00057	&-0.00501	&-	&-\\
$x_{55}$	&-	&-	&-	&-	&-\\
$x_{56}$	&-	&-0.02272	&0.00131	&-	&-\\
$x_{57}$	&-	&-	&-	&-	&-\\
$x_{58}$	&0.11223	&-	&0.05349	&-	&-\\
$x_{59}$	&0.23868	&-0.00711	&0.07830	&-	&-\\
$x_{60}$	&-0.07807	&-0.05727	&-	&-	&-0.02819\\
$x_{61}$	&-0.06729	&-	&-	&-	&-\\
$x_{62}$	&-0.02121	&-	&-	&-	&-\\
$x_{63}$	&-0.01886	&0.04294	&0.00548	&-	&-\\
$x_{64}$	&-0.01265	&0.02236	&-	&-	&-\\
$x_{65}$	&-	&-	&-	&-	&-\\
$x_{66}$	&-	&-	&-	&-	&-\\
$x_{67}$	&-0.03609	&-	&-	&-	&-\\
$x_{68}$	&-0.07929	&0.01190	&-0.00001	&-	&-\\
$x_{69}$	&-	&-	&-	&-	&-\\
$x_{70}$	&0.09774	&-0.01388	&0.01683	&-	&-\\
$x_{71}$	&-	&-	&-	&-	&-\\
$x_{72}$	&-0.08791	&-	&-	&-	&-\\
$x_{73}$	&-0.06590	&-0.13467	&0.03526	&-	&-\\
$x_{74}$	&0.05718	&-	&-	&-	&-\\
$x_{75}$	&-0.14786	&-0.03133	&-	&-	&-\\
$x_{76}$	&-0.12865	&-0.07620	&-0.09485	&-	&-\\
$x_{77}$	&0.04578	&0.04694	&-	&-	&-\\
$x_{78}$	&0.01510	&0.01860	&0.08887	&-	&-\\
$x_{79}$	&-0.00755	&0.00441	&0.01526	&-	&-0.56947\\
$x_{80}$	&-0.06835	&-	&-	&-	&-\\
$x_{81}$	&-	&-	&-0.00166	&-	&-\\
$x_{82}$	&-0.07267	&-	&-	&-	&-\\
$x_{83}$	&-0.00061	&0.02782	&-	&-	&-\\
$x_{84}$	&-	&-	&-	&-	&-\\
$x_{85}$	&-	&-	&-	&-	&-\\
$x_{86}$	&-0.02223	&0.02194	&0.03417	&-	&-\\
$x_{87}$	&0.00029	&-	&-	&-	&-\\
$x_{88}$	&-	&-	&-	&-	&-\\
$x_{89}$	&-0.06311	&0.03682	&-0.00977	&-	&-\\
$x_{90}$	&-	&-	&-	&-	&-\\
$x_{91}$	&-	&-	&-	&-	&-\\
$x_{92}$	&0.06938	&-0.03040	&-0.00542	&-	&-\\
$x_{93}$	&-	&-	&-	&-	&-\\
$x_{94}$	&0.05246	&-	&-0.00793	&-	&-\\
$x_{95}$	&-0.01214	&-	&-0.00345	&-	&-\\
$x_{96}$	&-	&-0.06544	&-0.00007	&-	&-\\
$x_{97}$	&0.03763	&-	&-	&-	&-\\
$x_{98}$	&0.04560	&0.04346	&0.00717	&-	&-\\
$x_{99}$	&0.03892	&-	&-0.01578	&-	&-\\
$x_{100}$	&0.01633	&-	&-0.01509	&-	&-\\
$x_{101}$	&0.04869	&0.01218	&0.00076	&-	&-\\
$x_{102}$	&-0.01996	&-	&-	&-	&-\\
$x_{103}$	&-	&-	&-	&-	&-\\
$x_{104}$	&-	&-	&-	&-	&-\\
$x_{105}$	&-0.00248	&-	&-	&-	&-\\
$x_{106}$	&-0.00344	&-0.03221	&0.01461	&-	&-\\
$x_{107}$	&-0.00779	&-0.01415	&0.00106	&-	&-\\
\hline
\end{tabular}
\caption{Fitted model parameters for the subset of the residential building data (part 2).}\label{ResizeRBData2}
\end{center}
\end{table}
From these Tables we can see that, in this high-dimensional setting, we still obtain acceptable results for the regularized MoE models and the EM algorithm using the proximal Newton method is a good tool for the parameter estimation. 
%
The running time in this experiment is about only few ($\sim 8$) minutes  and the algorithm is quite effective in this setting. 

\section{Conclusion and future work}\label{Sec:Con}
In this paper, we proposed a regularized MLE for the MoE model which encourages sparsity, and developed two versions of a blockwise EM algorithm to monotonically maximize this regularized objective towards at least a local maximum.
The proposed regularization does not require using approximations as in standard MoE regularization. 
The proposed algorithms are based on univariate updates of the model parameters via and MM and coordinate ascent, which allows to tackle matrix inversion problems and obtain sparse solutions.
The results 
in terms of
parameter estimation, the estimation of the actual support of the sparsity, and clustering accuracy,
obtained on simulated and three real data sets, 
confirm the effectiveness of our proposal at least for problems of moderate dimension. Namely, the model sparsity does not include significant bias in terms of parameter estimation nor in terms of recovering the actual clusters of the heterogeneous data. 
The obtained models with the proposed approach are sparse which promote its scalability to high-dimensional problems.
The hybrid EM/MM algorithm is a potential approach. However, this model should be considered carefully, especially for non-smooth penalty functions. The coordinate ascent approach for maximizing the M-step, however, works quite well although, while we do not have the closed form update in this situation.
A proximal Newton extension is possible to obtain closed form solutions for an approximate of the M-step as an efficient method that is promoted to deal with high-dimensional data sets. 
First experiments on an example of a quite high-dimensional scenario with a subset of real data containing $90$ observations and $108$ features provide encouraging results. 
%
A future work will consist in investigating more the high-dimensional setting, and performing additional model selection experiments as well as considering hierarchical MoE and MoE for discrete data. 

\section*{\label{Appendix}Appendix}
The proposed EMM algorithm maximizes the penalised log-likelihood function (\ref{eq:PenLoglik MoE}).
To show that the penalized log-likelihood is monotonically improved, that is 
\begin{equation}
PL(\bstheta^{[q+1]}) \geq PL(\bstheta^{[q]}),
\end{equation}
 we need to show that 
\begin{eqnarray}
Q(\bstheta^{[q+1]},\bstheta^{[q]}) \geq Q(\bstheta^{[q]},\bstheta^{[q]}).
\label{eq: Q(theta q+1)>=Q(theta q)}
\end{eqnarray}
Indeed, as in the standard EM algorithm algorithm for the non-penalised maximum likelihood estimation, by applying Bayes theorem we have
\begin{eqnarray}
\log PL(\bstheta)=  \log PL_c (\bstheta) - \log p(\bz|\cD;\bstheta), 
 \end{eqnarray}
and by taking the conditional expectation with respect to the latent variables $\bz$, given the observed data $\cD$ and the current parameter estimation $\bstheta^{[q]}$, the conditional expectation of the penalised completed-data log-likelhood is given by:
\begin{eqnarray}
 \E \left[\log PL(\bstheta)|\cD,\bstheta^{[q]}\right] = \E \left[\log PL_c (\bstheta) |\cD,\bstheta^{[q]}\right] - \E \left[\log p(\bz|\cD;\bstheta) |\cD,\bstheta^{[q]}\right].
\label{eq: Expected comp loglik decomposition}
\end{eqnarray} 
Since the penalised log-likelihood function $\log PL(\bstheta)$ does not depend on the variables $\bz$, its expectation with respect to $\bz$ therefore still unchanged and we get the following relation: 
\begin{equation}
\log PL(\bstheta) = \underbrace{\E \left[\log PL_c (\bstheta)|\cD,\bstheta^{[q]}\right]}_{Q(\bstheta,\bstheta^{[q]})} - \underbrace{\E \left[\log p \big(\bz|\cD;\bstheta \big)|\cD,\bstheta^{[q]}\right]}_{H(\bstheta,\bstheta^{[q]})}.
\label{eq: decomposition of L into Q and H}
\end{equation}
Thus, the value of change of the 
penalised log-likelihood function 
between two successive iterations is given by:%
{\footnotesize 
\begin{eqnarray}
\log PL(\bstheta^{[q+1]}) - \log PL(\bstheta^{[q]}) =  \Big(Q(\bstheta^{[q+1]},\bstheta^{[q]}) - Q(\bstheta^{[q]},\bstheta^{[q]})\Big)  - \Big( H(\bstheta^{[q+1]},\bstheta^{[q]}) - H(\bstheta^{[q]},\bstheta^{[q]}) \Big). 
\label{eq: L[q+1] - L[q]}
\end{eqnarray}
}As in the standard EM algorithm, it can be easily shown, by using Jensen' inequality,  that the second term $H(\bstheta^{[q+1]},\bstheta^{[q]}) - H(\bstheta^{[q]},\bstheta^{[q]})$ in the r.h.s of (\ref{eq: L[q+1] - L[q]}) is negative
and we therefore just need to show that the first term 
$Q(\bstheta^{[q+1]},\bstheta^{[q]}) - Q(\bstheta^{[q]},\bstheta^{[q]})$ is positive.


In the following, we show that 
$Q(\bstheta^{[q+1]},\bstheta^{[q]}) \geq Q(\bstheta^{[q]},\bstheta^{[q]})$.
First, the $Q$-function is decomposed as
\begin{equation}
Q(\bstheta;\bstheta^{[q]}) = Q(\bsw;\bstheta^{[q]}) + Q(\{\bsbeta_k, \sigma_k^2\};\bstheta^{[q]})
\end{equation}
and is accordingly maximized separately w.r.t. $\bsw$, $\{\bsbeta_k\}$ and $\{\sigma_k^2\}$. 

To update $\bsw$, first we use a univariate MM algirthm to iteratively maximize the minorizing function
%
$G(\bsw|\bsw^{[q]})$ for $Q(\bsw;\bstheta^{[q]})$ 
which satisfies
\begin{equation}
Q(\bsw;\bstheta^{[q]}) \ge    G(\bsw|\bsw^{[q]}), \forall \bsw
\end{equation}
and 
\begin{equation}
Q(\bsw^{[q]};\bstheta^{[q]})= G(\bsw^{[q]}|\bsw^{[q]}).
\end{equation} In our situation, the minorizing function is concave and has a separate structure. We thus use a one-dimensional Newton Raphson algorithm to maximize it. 
Thus, the solution $\bsw^{[q+1]}$ guarantees
\begin{equation}
G(\bsw^{[q+1]}|\bsw^{[q]})\ge G(\bsw^{[q]}|\bsw^{[q]})
\end{equation}
and hence we have  
\begin{equation}
Q(\bsw^{[q+1]};\bstheta^{[q]})\ge G(\bsw^{[q+1]}|\bsw^{[q]})\ge G(\bsw^{[q]}|\bsw^{[q]}) =  Q(\bsw^{[q]};\bstheta^{[q]}).
\label{eq: MM update Q(w q+1)>=Q(w q)}
\end{equation}
Hence, the MM algorithm leads to the improvement of the value of the $Q(\bsw;\bstheta^{[q]})$ function.

For the second version of the EM algorithm which uses the coordinate ascent algorithm to update $\bsw$, we rely on the work of \cite{Tse88} and \cite{Tse01}, where it is proved that, if the nonsmooth part of $Q(\bsw;\bstheta^{[q]})$ has a separate structure, the coordinate ascent algorithm is successful in finding the 
$\bsw^{[q+1]} = \arg\max\limits_{\bsw}Q(\bsw;\bstheta^{[q]})$. 
At each step of the coordinate ascent algorithm, within the M-step of the EM algorithm, we  iteratively update the $j$th component, while fixing the other parameters to their previous values:
\begin{equation}
w_{kj}^{[q,s+1]} = \arg\max\limits_{w_{kj}} Q(w_{kj};\bstheta^{[q,s]}),
\end{equation}
$s$ being the current iteration of the coordinate ascent algorithm. 
The function $Q(w_{kj};\bstheta^{[q]})$ is concave, and the used iterative procedure to find $w_{kj}^{[q+1]}$ is the Newton Raphson algorithm. 
Hence, the coordinate ascent leads to the improvement of the function $Q(\bsw;\bstheta^{[q]})$, that is
\begin{equation}
Q(\bsw^{[q+1]};\bstheta^{[q]})\ge Q(\bsw^{[q]};\bstheta^{[q]}).
\label{eq: CA update Q(w q+1)>=Q(w q)}
\end{equation} 

The updates of the experts' parameters $\{\bsbeta\}$ and $\{\sigma^2 \}$ are performed by separate maximizations of $Q(\bsbeta,\sigma^2;\bstheta^{[q]})$. This function is concave and has the quadratic form. Hence, the coordinate ascent algorithm with soft-thresholding operator is successful to provide the updates
\begin{equation}
\bsbeta^{[q+1]} = \arg\max\limits_{\bsbeta} Q(\bsbeta,\sigma^{[q]};\bstheta^{[q]}),
\end{equation}
and \begin{equation}
\sigma^{[q+1]} = \arg\max_{\sigma} Q(\bsbeta^{[q+1]},\sigma;\bstheta^{[q]})
\end{equation}
and thus we have 
\begin{equation}
Q(\bsbeta^{[q+1]};\bstheta^{[q]})\ge Q(\bsbeta;\bstheta^{[q]}) \ge Q(\bsbeta^{[q]};\bstheta^{[q]}),
\label{eq: EMM update Q(beta q+1)>=Q(theta q)}
\end{equation}
and
\begin{equation}
Q(\sigma^{[q+1]};\bstheta^{[q]})\ge Q(\sigma;\bstheta^{[q]}) \ge Q(\sigma^{[q]};\bstheta^{[q]}).
\label{eq: EMM update Q(sigma q+1)>=Q(theta q)}
\end{equation} \\

Equations 
(\ref{eq: MM update Q(w q+1)>=Q(w q)}),
(\ref{eq: CA update Q(w q+1)>=Q(w q)}),
(\ref{eq: EMM update Q(beta q+1)>=Q(theta q)}), and
(\ref{eq: EMM update Q(sigma q+1)>=Q(theta q)})
show that $(\ref{eq: Q(theta q+1)>=Q(theta q)})$ holds, and hence the penalised log-likelihood in monotonically increased by the proposed algorithm.  

\section*{Acknowledgements}
The authors would like to thank the R\'egion Normandie for the financial support of this research via the research project RIN AStERiCs. 
The authors would also like to very much thank the anonymous reviewers and the editor for their comments who helped to improve the manuscript.

\bibliography{RMoE}


\end{document}